%% file: noise_sgd.tex
\documentclass{article}
\usepackage[11pt]{moresize}
\usepackage{amssymb}
\usepackage{amsmath}
\usepackage{amsthm}
\usepackage{enumerate}
\usepackage{enumitem}
\usepackage{microtype}
\usepackage{graphicx}
\usepackage{dsfont} 
\usepackage{adjustbox}
\usepackage{float}
\usepackage{framed}

\usepackage{subfigure}
\usepackage{booktabs} 

\usepackage{hyperref}

\usepackage[accepted]{icml2018}

\input{defs}

\icmltitlerunning{Escaping Saddles with Stochastic Gradients}

\begin{document}

\twocolumn[
\icmltitle{Escaping Saddles with Stochastic Gradients}



\icmlsetsymbol{equal}{*}

\begin{icmlauthorlist}
\icmlauthor{Hadi Daneshmand}{equal,eth}
\icmlauthor{Jonas Kohler}{equal,eth}
\icmlauthor{Aurelien Lucchi}{eth}
\icmlauthor{Thomas Hofmann}{eth}

\end{icmlauthorlist}

\icmlaffiliation{eth}{ETH, Zurich, Switzerland}

\icmlcorrespondingauthor{Hadi Daneshmand}{hadi.daneshmand@inf.ethz.ch}

\icmlkeywords{Machine Learning, ICML}

\vskip 0.3in
]



\printAffiliationsAndNotice{\icmlEqualContribution} 


\begin{abstract}
We analyze the variance of stochastic gradients along negative curvature directions in certain non-convex machine learning models and show that stochastic gradients exhibit a strong component along these directions. Furthermore, we show that - contrary to the case of isotropic noise - this variance is proportional to the magnitude of the corresponding eigenvalues and not decreasing in the dimensionality. Based upon this observation we propose a new assumption under which we show that the injection of explicit, isotropic noise usually applied to make gradient descent escape saddle points can successfully be replaced by a simple SGD step.
Additionally - and under the same condition -  we derive the first convergence rate for plain SGD to a \textit{second-order} stationary point in a number of iterations that is independent of the problem dimension. 

\end{abstract}

\input{01_introduction}
\input{02_related_work}
\input{03_pgd_revised}

\input{04_sgd_revised_analysis}
\input{05_ncc_sgd}

\input{06_experiments}
\input{07_conclusion}
\newpage
\section*{Acknowledgements} We would like to thank Antonio Orvieto, Yi Xu, Tianbao Yang, Kaiqing Zhang  and Alec Koppel for pointing out mistakes in the draft and their help in improving the result. We also thank Kfir Levy, Gary Becigneul, Yannic Kilcher and Kevin Roth for their helpful discussions. 
\bibliography{noise_sgd}
\bibliographystyle{icml2018}

\input{08_appendix}

\end{document}

%% file: defs.tex
\newcommand{\tr}{\text{Tr}}

\newcommand{\gf}{\nabla f}
\newcommand{\gft}{{\nabla \tilde{f}}}
\renewcommand{\gg}{\nabla g}
\newcommand{\hf}{\nabla^2 f} 

\newcommand{\bdelta}{\boldsymbol{\delta}}

\newcommand{\w}{{\bf w}}
\renewcommand{\u}{{\bf u}}
\renewcommand{\d}{{\bf d}}
\newcommand{\pw}{\tilde{\w}}

\newcommand{\x}{{\bf x}}

\renewcommand{\v}{{\bf v}}
\newcommand{\z}{{\bf z}}

\newcommand{\bigo}{{\cal O}}
\newcommand{\tbigo}{\tilde{{\cal O}}}

\usepackage{mathtools}

\def\Im{{\bf I}}

\newcommand{\A}{{\mathcal A}}

\newcommand{\E}{{\mathbf E}}

\newcommand{\F}{{\mathcal F}}

\newcommand{\R}{{\mathbb{R}}}
\renewcommand{\S}{{\mathcal S}}
\renewcommand{\H}{{\cal H}}
\renewcommand{\L}{{\cal L}}

\renewcommand{\P}{{\cal P}}
 
\renewcommand{\F}{{\cal F}}
\newcommand{\bzeta}{\boldsymbol{\zeta}}

\newcommand{\gt}{g_{\text{thres}}}
\newcommand{\ft}{f_{\text{thres}}}
\renewcommand{\tr}{{t_{\text{thres}}}}
\newcommand{\tn}{t_{\text{noise}}}

\newtheorem{theorem}{Theorem}
\newtheorem{lemma}{Lemma}

\newtheorem{assumption}{Assumption}
\theoremstyle{definition}
\newtheorem{definition}{Definition}

\usepackage{eqparbox}
\renewcommand{\algorithmiccomment}[1]{\hfill\eqparbox{COMMENT}{\# #1}}

%% file: 01_introduction.tex

\section{Introduction}

In this paper we analyze the use of gradient descent (GD) and its stochastic variant (SGD) to minimize objectives of the form
\begin{equation} 
\w^* = \arg \min_{\w \in \R^d}  \left [ f(\w) := \E_{\z\sim\P} \left[  f_\z(\w) \right] \right],
\label{eq:f_x}
\end{equation}
where $f \in C^2(\R^d,\R)$ is a not necessarily convex loss function and $\P$ is an arbitrary probability distribution.

In the era of big data and deep neural networks, (stochastic) gradient descent is a core component of many training algorithms~\cite{bottou2010large}. What makes SGD so attractive is its simplicity, its seemingly universal applicability and a convergence rate that is independent of the size of the training set. One specific trait of SGD is the inherent noise, originating from sampling training points, whose variance has to be controlled in order to guarantee convergence either through a conservative step size~\cite{nesterov2013introductory} or via explicit variance-reduction techniques~\cite{johnson2013accelerating}.

While the convergence behavior of SGD is well-understood for convex functions~\cite{bottou2010large}, we are here interested in the optimization of non-convex functions which pose additional challenges for optimization in particular due to the presence of saddle points and suboptimal local minima~\cite{dauphin2014identifying, choromanska2015loss}. For example, finding the global minimum of even a degree 4 polynomial can be NP-hard~\cite{hillar2013most}.  Instead of aiming for a global minimizer, a more practical goal is to search for a local optimum of the objective. In this paper we thus focus on reaching a second-order stationary point of smooth non-convex functions.
Formally, we aim to find an  $(\epsilon_g, \epsilon_h)$-second-order stationary point $\w$ such that the following conditions hold:
\begin{equation}
\| \gf(\w) \| \leq \epsilon_g \quad\text{and}\quad \hf(\w) \succcurlyeq -\epsilon_h \Im,
\label{eq:second_order_stat_point}
\end{equation}
where $\epsilon_g, \epsilon_h > 0$.

Existing work, such as~\cite{ge2015escaping, jin2017escape}, proved convergence to a point satisfying Eq.~\eqref{eq:second_order_stat_point} for modified variants of gradient descent and its stochastic variant by requiring additional noise to be explicitly added to the iterates along the entire path (former) or whenever the gradient is sufficiently small (latter). Formally, this yields the following update step for the perturbed GD and SGD versions:
\begin{align}
\label{eq:PGD_update}
&\text{PGD: }\:\; \w_{t+1} = \w_t  - \eta_t \gf(\w_t) + r \zeta_{t+1} \\
\label{eq:PSGD_update}
&\text{PSGD: } \w_{t+1} = \w_t - \eta_t \left( \gf_\z(\w_t)  + \zeta_t \right),
\end{align}
where $\zeta_t$ is typically zero-mean noise sampled uniformly from a unit sphere.

\paragraph{Isotropic noise} The perturbed variants of GD and SGD in Eqs.~\eqref{eq:PGD_update}-\eqref{eq:PSGD_update} have been analyzed for the case where the added noise $\zeta_t$ is isotropic~\cite{ge2015escaping, levy2016power, jin2017escape} or at least exhibits a certain amount of variance along all directions in $\R^d$~\cite{ge2015escaping}. As shown in Table \ref{table:dim_dep}, an immediate consequence of such conditions is that they introduce a dependency on the input dimension $d$ in the convergence rate. Furthermore, it is unknown as of today, if this condition is satisfied by the intrinsic noise of vanilla SGD for any specific class of machine learning models. Recent empirical observations show that this is not the case for training neural networks~\cite{chaudhari2017stochastic}.
\begin{center} 
\begin{table*}[h]
\centering
\label{table:dim_dep}
\begin{tabular}{llll}
\hline
Algorithm                & First-order Complexity  & Second-order Complexity     & $d$ Dependency \\\hline
Perturbed SGD \cite{ge2015escaping} & $\bigo(d^p \epsilon^{-4}_g) $  & $\bigo(d^p \epsilon^{-16}_h)$ & poly     \\
SGLD  \cite{Zhang2017AHT}            & $\bigo(d^p \epsilon^{-2}_g)$   &$\bigo(d^p \epsilon^{-4}_h)$  & poly        \\
PGD \cite{jin2017escape}       & $\bigo(\log^4(d/\epsilon_g)\epsilon^{-2}_g)$ & $\bigo(\log^4(d/\epsilon_h) \epsilon^{-4}_h)$ & poly-log     \\
SGD+NEON~\cite{xu2017first} & $\tbigo(\epsilon^{-4}_g)$ & $\tbigo(\epsilon^{-8}_g)$ & poly-log \\
\hline  
CNC-GD (Algorithm \ref{alg:CNC-PGD}) &$\bigo(\epsilon^{-2}_g\log(1/\epsilon_g))$ & $\bigo(\epsilon^{-5}_h\log(1/\epsilon_h))$ & free \\
CNC-SGD (Algorithm \ref{alg:cnc_sgd}) &$\bigo(\epsilon^{-10}_g\log^2(1/\epsilon_g))$ & $\bigo(\epsilon^{-10}_h\log^2(1/\epsilon_h))$ & free \\
\end{tabular} 
\caption{Dimension dependency and iteration complexity to reach a second-order stationary point as characterized in Eq.~\eqref{eq:second_order_stat_point}. The notation $\bigo(\cdot)$ hides constant factors and $\tbigo(\cdot)$ hides a poly-logarithmic factor.}
\end{table*}
\end{center}
\vspace{-1cm}
In this work, we therefore turn our attention to the following question. Do we need to perturb iterates along \textit{all} dimensions in order for (S)GD to converge to a second-order stationary point? Or is it enough to simply rely on the inherent variance of SGD induced by sampling? More than a purely theoretical exercise, this question has some very important practical implications since in practice the vast majority of existing SGD methods do not add additional noise and therefore do not meet the requirement of isotropic noise. Thus we instead focus our attention on a less restrictive condition for which perturbations only have a guaranteed variance along directions of negative curvature of the objective, i.e. along the eigenvector(s) associated with the minimum eigenvalue of the Hessian. Instead of explicitly adding noise as done in Eqs.~\eqref{eq:PGD_update} and \eqref{eq:PSGD_update}, we will from now on consider the simple SGD step:
\begin{align} 
\w_{t+1} = \w_{t} - \eta \gf_\z(\w_t)
\label{eq:pertubed_steps}
\end{align}

and propose the following sufficient condition on the stochastic gradient $\gf_\z(\w)$ to guarantee convergence 
to a second-order stationary point.
\begin{assumption}[Correlated Negative Curvature (CNC)]
Let $\v_\w$ be the eigenvector corresponding to the minimum eigenvalue of the Hessian matrix $\hf(\w)$. The stochastic gradient $\gf_\z(\w)$ satisfies the CNC assumption, if the second moment of its projection along the direction $\v_\w$ is uniformly bounded away from zero, i.e.
\begin{equation} 
\exists \gamma>0 \,\,\text{s.t.}\,\, \forall \w: \E [\langle \v_\w, \gf_\z(\w) \rangle^2] > \gamma\,.
\label{eq:CNC}
\end{equation}
\label{ass:CNC}
\end{assumption}
\vspace{-0.75cm}
\paragraph{Contributions}
Our contribution is fourfold: First, we analyze the convergence of GD perturbed by SGD steps (Algorithm \ref{alg:CNC-PGD}). Under the CNC assumption, we demonstrate that this method converges to an $(\epsilon,\epsilon^{2/5})$-second-order stationary point in $\tbigo(\epsilon^{-2})$ iterations and with high probability. Second, we prove that vanilla SGD as stated in Algorithm \ref{alg:cnc_sgd} -again under Assumption \ref{ass:CNC}- also convergences to an $(\epsilon,\epsilon)$-second-order stationary point in $\tbigo(\epsilon^{-10})$ iterations and with high probability. To the best of our knowledge, this is the first second-order convergence result for SGD without adding additional noise. One important consequence of not relying on isotropic noise is that the rate of convergence becomes independent of the input dimension $d$. This can be a very significant practical advantage when optimizing deep neural networks that contain millions of trainable parameters.
Third, we prove that stochastic gradients satisfy Assumption \ref{ass:CNC} in the setting of learning half-spaces, which is ubiquitous in machine learning. Finally, we provide experimental evidence suggesting the validity of this condition for training neural networks. In particular we show that, while the variance of uniform noise along eigenvectors corresponding to the most negative eigenvalue decreases as $\bigo(1/d)$, stochastic gradients have a significant component along this direction independent of the \textit{width} and \textit{depth} of the neural net. When looking at the entire eigenspectrum, we find that this variance increases with the magnitude of the associated eigenvalues. Hereby, we contribute to a better understanding of the success of training deep networks with SGD and its extensions.

%% file: 02_related_work.tex

\section{Background \& Related work}

\paragraph{Reaching a 1st-order stationary point} For smooth non-convex functions, a first-order stationary point satisfying $\| \nabla f(\x) \| \leq \epsilon$ can be reached by GD and SGD in $\bigo(\epsilon^{-2})$ and $\bigo(\epsilon^{-4})$ iterations respectively~\cite{nesterov2013introductory}. Recently, it has been shown that GD can be accelerated to find such a point in $\bigo(\epsilon^{-7/4}\log(\epsilon^{-1})$)  \cite{carmon2017convex}.

\paragraph{Reaching a 2nd-order stationary point}
In order to reach second-order stationary points, existing first-order techniques rely on explicitly adding isotropic noise with a known variance (see Eq.~\eqref{eq:PGD_update}). The key motivation for this step is the insight that the area of attraction to a saddle point constitutes an unstable manifold and thus gradient descent methods are unlikely to get stuck, but if they do, adding noise allows them to escape \cite{lee2016gradient}. Based upon this observations, recent works prove second-order convergence of normalized GD \cite{levy2016power} and perturbed GD \cite{jin2017escape}. The later needs at most $\bigo(\max\{\epsilon_g^{-2},\epsilon_h^{-4}\}\log^4(d))$ iterations and is thus the first to achieve a poly-log dependency on the dimensionality. The convergence of SGD with additional noise was analyzed in~\cite{ge2015escaping} but to the best of our knowledge, no prior work demonstrated convergence of SGD \textit{without} explicitly adding noise.

\paragraph{Using curvature information}
Since negative curvature signals potential descent directions, it seems logical to apply a second-order method to exploit this curvature direction in order to escape saddle points. Yet, the prototypical Newton's method has no global convergence guarantee and is locally attracted by saddle points and even local maxima~\cite{dauphin2014identifying}. Another issue is the computation (and perhaps storage) of the Hessian matrix, which requires $\bigo(nd^2)$ operations as well as computing the inverse of the Hessian, which requires $\bigo(d^3)$ computations.

The first problem can be solved by using trust-region methods that guarantee convergence to a second-order stationary point~\cite{conn2000trust}. Among these methods, the Cubic Regularization technique initially proposed by~\cite{nesterov2006cubic} has been shown to achieve the optimal worst-case iteration bound $\bigo(\max\{\epsilon_g^{-3/2},\epsilon_h^{-3}\})$ \cite{cartis2012much}. The second problem can be addressed by replacing the computation of the Hessian by Hessian-vector products that can be computed efficiently in $\bigo(nd)$~\cite{pearlmutter1994fast}. This is applied e.g. using matrix-free Lanczos iterations~\cite{curtis2017exploiting, reddi2017generic} or online variants such as Oja's algorithm~\cite{allen2017natasha}. Sub-sampling the Hessian can furthermore reduce the dependence on $n$ by using various sampling schemes~\cite{kohler2017sub,xu2017newton}. Finally, ~\cite{xu2017first} and ~\cite{allen2017neon2} showed that noisy gradient updates act as a noisy Power method allowing to find a negative curvature direction using only first-order information. 
Despite the recent theoretical improvements obtained by such techniques, first-order methods still dominate for training large deep neural networks. Their theoretical properties are however not perfectly well understood in the general case and we here aim to deepen the current understanding.

%% file: 03_pgd_revised.tex

\section{GD Perturbed by Stochastic Gradients}

In this section we derive a converge guarantee for a combination of gradient descent and stochastic gradient steps, as presented in Algorithm~\ref{alg:CNC-PGD}, for the case where the stochastic gradient sequence meets the CNC assumption introduced in Eq.~(\ref{eq:CNC}).
We name this algorithm CNC-PGD since it is a modified version of the PGD method~\cite{jin2017escape}, but use the intrinsic noise of SGD instead of requiring noise isotropy. 
Our theoretical analysis relies on the following smoothness conditions on the objective function $f$.
\begin{assumption}[Smoothness Assumption] \label{ass:smoothness}
We assume that the function $f \in C^2(\R^d,\R)$ has $L$-Lipschitz gradients and $\rho$-Lipschitz Hessians and that each function $f_\z$ has an $\ell$-bounded gradient.\footnote{See Appendix A for formal definitions.} W.l.o.g. we further assume that $\rho$, $\ell$, and $L$ are greater than one.
\end{assumption}
Note that $L$-smoothness and $\rho$-Hessian Lipschitzness are standard assumptions for convergence analysis to a second-order stationary point \cite{ge2015escaping,jin2017escape,nesterov2006cubic}. The boundedness of the stochastic gradient $\gf_\z(\w)$ is often used in stochastic optimization~\cite{moulines2011non}.
\begin{algorithm}[tb]
\begin{algorithmic}[1]
   \STATE \textbf{Input:} $\gt$, $\tr$, $T$, $\eta$ and $r$ \;
   \STATE $\tn \leftarrow -\tr -1 $
   \FOR{$t=1,2, \dots,T$}
   \IF{$\| \gf(\w_t)\|^2 \leq \gt \text{ and } t-\tn\geq \tr$}
   \STATE $\pw_t \leftarrow \w_t, \tn \leftarrow t$ \qquad\quad\;\;\textit{\footnotesize{\# used in the analysis}}\\
   \STATE $\w_{t+1} \leftarrow \w_t - r \gf_\z(\w_t)$ \;\;\;\; \textit{\footnotesize{\#} $z \stackrel{\text{i.i.d}}{\sim} \P$}
   \ELSE
   \STATE $\w_{t+1} \leftarrow \w_{t} - \eta \gf(\w_t) $
   \ENDIF
   \ENDFOR
   \STATE \textbf{return }$\widehat{\w}$ uniformly from $\{\w_1,\dots, \w_T \}$
\end{algorithmic}
\caption{CNC-PGD}
\label{alg:CNC-PGD}
\end{algorithm}

\paragraph{Parameters} 
The analysis presented below relies on a particular choice of parameters. Their values are set based on the desired accuracy $\epsilon$ and presented in Table~\ref{tab:pgd_parameters}.
\begin{table}[tb]
\footnotesize
    \centering
    \begin{tabular}{c c c}
    Parameter  & Value & Dependency on
    $\epsilon$ \\
    \hline
    $\eta$ & $1/L$ & Independent\\ 
    $r$ & $c_1(\delta \gamma \epsilon^{4/5})/(\ell^3 L^2)$ & $\bigo(\epsilon^{4/5})$ \\ 
    $\omega$ & $\log(\ell L/(\gamma\delta\epsilon))$ & $\bigo(\log(1/\epsilon))$ \\
    $\tr$ & $c_2 L(\sqrt{\rho} \epsilon^{2/5})^{-1} \omega $ & $\bigo(\epsilon^{-2/5} \log(1/\epsilon))$ \\ 
     $\ft$ & $c_3 \delta \gamma^2  \epsilon^{8/5}/(\ell^2 L)^2$ & $\bigo(\epsilon^{8/5})$
     \\
    $\gt$ & $\ft/\tr$ & $\bigo(\epsilon^2/\log(1/\epsilon))$ \\
    $T$ & $4 (f(\w_0) - f^*)/(\eta \delta \gt)$ & $\bigo(\epsilon^{-2}\log(1/\epsilon))$ \\ 
    \hline
\end{tabular}
    \caption{\textit{Parameters of CNC-PGD.} Note that the constants $\ft$ and $\omega$ are only needed for the analysis and thus not required to run Algorithm~\ref{alg:CNC-PGD}. The constant $\delta \in (0,1)$ comes from the probability statement in Theorem~\ref{theorem:pgd}. Finally the constants $c_1,c_2$ and $c_3$ are independent of the parameters $\gamma$,$\delta$, $\epsilon$, $\ell$, $\rho$, and $L$ (see Appendix B for more details).}
    \label{tab:pgd_parameters}
\end{table}
\subsection{PGD Convergence Result}
\begin{theorem} \label{theorem:pgd} 
Let the stochastic gradients $\gf_\z(\w_t)$ in CNC-PGD satisfy Assumption \ref{ass:CNC} and let $f$, $f_\z$ satisfy Assumption \ref{ass:smoothness}. Then Algorithm \ref{alg:CNC-PGD} returns an  $\left(\epsilon,\sqrt{\rho}\epsilon^{2/5}\right)$-second-order stationary point with probability at least $\left(1-\delta\right)$ after \[\bigo\left( (\ell L)^4(\delta \gamma \epsilon)^{-2}\log\left(\frac{\ell L}{\eta \delta \gamma\epsilon^{2/5}}\right)\right)\] steps, where $\delta<1$.
\end{theorem} 

\paragraph{Remark}
CNC-PGD converges polynomially to a second-order stationary point under Assumption \ref{ass:CNC}. By relying on isotropic noise, \cite{jin2017escape} prove convergence to a $\left(\epsilon,(\rho\epsilon)^{1/2}\right)$-stationary point in $\tbigo\left(1/\epsilon^2\right)$ steps. The result of Theorem~\ref{theorem:pgd} matches this rate in terms of first-order optimality but is worse by an $\epsilon^{-0.1}$-factor in terms of the second-order condition. Yet, we do not know whether our rate is the best achievable rate under the CNC condition and whether having isotropic noise is necessary to obtain a faster rate of convergence. As mentioned previously, a major benefit of employing the CNC condition is that it results in a convergence rate that does not depend on the dimension of the parameter space.\footnote{This result is not in conflict with the dimensionality-dependent lower bound established in \cite{simchowitz2017gap} since they make no initialization assumption as we do in Assumption \ref{ass:CNC} (CNC).} Furthermore, we believe that the dependency to $\gamma$ (Eq.~\eqref{eq:CNC}) can be significantly improved.


\subsection{Proof sketch of Theorem~\ref{theorem:pgd}} 

In order to prove Theorem~\ref{theorem:pgd}, we consider three different scenarios depending on the magnitude of the gradient and the amount of negative curvature. Our proof scheme is mainly inspired by the analysis of perturbed gradient descent~\cite{jin2017escape}, where a deterministic sufficient condition is established for escaping from saddle points (see Lemma 11). This condition is shown to hold in the case of isotropic noise. However, the non-isotropic noise coming from stochastic gradients is more difficult to analyze. Our contribution is to show that a less restrictive assumption on the perturbation noise still allows to escape saddle points. Detailed proofs of each lemma are provided in the Appendix.

\paragraph{Large gradient regime} When the gradient is large enough, we can invoke existing results on the analysis of gradient descent for non-convex functions~\cite{nesterov2013introductory}. 
\begin{lemma} \label{lemma:decrease_f}
Consider a gradient descent step $\w_{t+1}=\w_t-\eta \gf(\w_t)$ on a $L$-smooth function $f$. For $\eta \leq 1/L$ this yields the following function decrease:
\begin{align} 
f(\w_{t+1}) - f(\w_t) \leq - \frac{\eta}{2} \| \gf(\w_t)\|^2.\end{align} 
\end{lemma} 
Using the above result, we can guarantee the desired decrease whenever the norm of the gradient is large enough. Suppose that $\| \gf(\w_t) \|^2 \geq \gt$, then Lemma~\ref{lemma:decrease_f} immediately yields 
\begin{align} 
f(\w_{t+1}) - f(\w_t) & \leq - \frac{\eta}{2}  \gt. \label{eq:large_gradient_result}
\end{align} 

\paragraph{Small gradient and sharp negative curvature regime}
Consider the setting where the norm of the gradient is small, i.e. $\|\gf(\w_t) \|^2 \leq \gt$, but the minimum eigenvalue of the Hessian matrix is significantly less than zero, i.e. $\lambda_{\min} (\hf(\w))\ll 0$. In such a case, exploiting Assumption~\ref{ass:CNC} (CNC) provides a guaranteed decrease in the function value after $\tr$ iterations, in expectation. 
 \begin{lemma} \label{lemma:small_grad_large_nc}
Let Assumptions~\ref{ass:CNC} and~\ref{ass:smoothness} hold. Consider perturbed gradient steps (Algorithm \ref{alg:CNC-PGD} with parameters as in Table~\ref{tab:pgd_parameters}) starting from ${\pw}_t$ such that $\|\gf(\pw_t)\|^2 \leq \gt$. Assume the Hessian matrix $\hf(\pw_t)$ has a large negative eigenvalue, i.e. 
 \begin{align} \label{eq:lambda_min_bound}
\lambda_{\min} (\hf(\pw_t)) \leq - \sqrt{\rho}\epsilon^{2/5}.
 \end{align}
Then, after $\tr$ iterations the function value decreases as 
\begin{align} 
\E \left[ f(\w_{t+\tr}) \right] - f(\pw_t)  \leq - \ft,
\end{align}  where the expectation is over the sequence $\{\w_k\}_{t+1}^{t+\tr}$. 
 \end{lemma} 
 
 \paragraph{Small gradient with moderate negative curvature regime}
 Suppose that $\| \gf(\w_t) \|^2 \leq \gt$ and that the absolute value of the minimum eigenvalue of the Hessian is close to zero, i.e. we already reached the desired first- and second-order optimality. In this case, we can guarantee that adding noise will only lead to a limited increase in terms of expected function value.
 
 \begin{lemma} \label{lemma:small_grad_small_nc}
 Let Assumptions~\ref{ass:CNC} and~\ref{ass:smoothness} hold. Consider perturbed gradient steps (Algorithm \ref{alg:CNC-PGD} with parameters as in Table~\ref{tab:pgd_parameters}) starting from ${\pw}_t$ such that $\|\gf(\pw_t)\|^2 \leq \gt$. Then after $\tr$ iterations, the function value cannot increase by more than 
 \begin{align} 
 \E \left[ f(\w_{t+\tr})\right]  - f(\pw_t) \leq \frac{ \eta\delta \ft}{4}, 
 \end{align}
 where the expectation is over the sequence $\{\w_k\}_{t+1}^{t+\tr}$.
 \end{lemma} 

\paragraph{Joint analysis}

We now combine the results of the three scenarios discussed so far. Towards this end we introduce the set $\S$ as 
\begin{align} 
\S := \{ \w \in \R^d \; | \;  & \| \gf(\w) \|^2 \geq \gt \nonumber \\
  & \; \text{or} \; \lambda_{\min} \left(\hf(\w)\right) \leq -\sqrt{\rho} \epsilon^{2/5}\}. \nonumber
\end{align}
Each of the visited parameters $\w_t, t=1,\ldots, T$ 
constitutes a random variable. For each of these random variables, we define the event $\A_t := \{ \w_t \in \S \} $. 
When $\A_t$ occurs, the function value decreases in expectation. Since the number of  steps required in the analysis of the large gradient regime and the sharp curvature regime are different, we use an amortized analysis similar to~\cite{jin2017escape} where we consider the per-step decrease~\footnote{Note that the amortization technique is here used to simplify the presentation but all our results hold without amortization.}. Indeed, when the negative curvature is sharp, then Lemma~\ref{lemma:small_grad_large_nc} provides a guaranteed decrease in $f$ which - when normalized per step - yields
\begin{align} 
\frac{\E \left[ f(\w_{t+\tr})\right] - f(\pw_t) }{\tr} \leq - \frac{\ft}{\tr} = - \eta \gt. 
\end{align} 
The large gradient norm regime of Lemma \ref{lemma:decrease_f} guarantees a decrease of the same order and hence
\begin{align} \label{eq:conditional_decrease}
\E \left[ f(\w_{t+1}) - f(\w_t) \; | \; \A_t \right] \leq -\frac{\eta}{2} \gt
\end{align}
follows from combining the two results.
Let us now consider the case when $\A_t^c$ (complement of $\A_t$) occurs. Then the result of Lemma~\ref{lemma:small_grad_small_nc} allows us to bound the increase in terms of function value, i.e.
\begin{align} \label{eq:contional_increase}
\E \left[ f(\w_{t+1}) - f(\w_t) \; | \; \A_t^c \right] \leq   \frac{\eta\delta}{4}  \gt.
\end{align}

\paragraph{Probabilistic bound}
  
The results established so far have shown that \emph{in expectation} the function value decreases until the iterates reach a second-order stationary point, for which Lemma~\ref{lemma:small_grad_small_nc} guarantees that the function value does not increase too much subsequently.\footnote{Since there may exist degenerate saddle points which are second-order stationary but not local minima we cannot guarantee that PGD stays close to a second-order stationary point it visits. One could rule out degenerate saddles using the strict-saddle assumption introduced in~\cite{ge2015escaping}.} This result guarantees visiting a second-order stationary point in $T$ steps (see Table~\ref{tab:pgd_parameters}). Yet, certifying second-order optimality is slightly more intricate as one would need to know which of the parameters $\{\w_1, \dots,\w_T\}$ meets the required condition. One solution to address this problem is to provide a high probability statement as suggested in~\cite{jin2017escape} (see Lemma 10). We here follow a similar approach except that unlike the result of~\cite{jin2017escape} that relies on exact function values, our results are valid in expectation. Our solution is to establish a high probability bound by returning one of the visited parameters picked uniformly at random. This approach is often used in stochastic non-convex optimization~\cite{ghadimi2013stochastic}.

The idea is simple: If the number of steps is sufficiently large, then  the results of Lemma \eqref{lemma:decrease_f}-\eqref{lemma:small_grad_small_nc} guarantee that the number of times we visit a second-order stationary point is high. Let $R$ be a random variable that determines the ratio of $(\epsilon,\sqrt{\rho}\epsilon^{2/5})$-second-order stationary points visited through the optimization path $\{\w_t\}_{t=1,\ldots,T}$.
Formally,
\begin{align} 
R := \frac{1}{T} \sum_{t=1}^T \mathds{1} \left( \A_t^c \right),
\end{align}
where $\mathds{1}$ is the indicator function.
Let $\P_t$ denote the probability of event $\A_t$ and $1-\P_t$ be the probability of its complement $\A_t^c$.  The probability of returning a second-order stationary point is simply
\begin{align} 
\E \left[ R \right] = \frac{1}{T}\sum_{t=1}^T (1-\P_t).
\end{align}
Estimating the probabilities $\P_t$ is difficult due to the interdependence of the random variables $\w_t$. However, we can upper bound the sum of the individual $\P_t$'s. 
Using the law of total expectation and the results from Eq.~\eqref{eq:conditional_decrease} and~\eqref{eq:contional_increase}, we bound the expectation of the function value decrease as:
\begin{multline}
    \E\left[ f(\w_{t+1}) - f(\w_t) \right] 
    \\ \leq \eta \gt \left( \delta/2 -(1+\delta/2)\P_t \right)/2. 
\end{multline} 

Summing over $T$ iterations yields 
\begin{multline} 
\sum_{i=1}^T \E \left[ f(\w_{t+1}) \right]  - \E \left[ f(\w_t) \right]  \\ \leq  \eta \gt \left(\delta T/2- (1+\delta/2) \sum_{t=1}^T \P_t\right)/2,
\end{multline}
which, after rearranging terms, leads to the following upper-bound
\begin{align}
\frac{1}{T}\sum_{t=1}^T \P_t \leq \frac{\delta}{2} + \frac{2\left(f(\w_0) - f^*)\right)}{ T \eta \gt} \leq \delta.
\end{align}
Therefore, the probability that $\A_t^c$ occurs uniformly over $\{1,\dots,T\}$ is lower bounded as 
\begin{align} 
\frac{1}{T}\sum_{t=1}^T (1-\P_t) \geq 1-\delta, 
\end{align}
which concludes the proof of Theorem~\ref{theorem:pgd}.

%% file: 04_sgd_revised_analysis.tex

\section{SGD without Perturbation}

We now turn our attention to the stochastic variant of gradient descent under the assumption that the stochastic gradients fulfill the CNC condition (Assumption \ref{ass:CNC}). We name this method CNC-SGD and demonstrate that it converges to a second-order stationary point without any additional perturbation. Note that in order to provide the convergence guarantee, we periodically enlarge the step size through the optimization process, as outlined in Algorithm~\ref{alg:cnc_sgd}. This periodic step size increase amplifies the variance along eigenvectors corresponding to the minimum eigenvalue of the Hessian, allowing SGD to exploit the negative curvature in the subsequent steps (using a smaller step size). Increasing the step size is therefore similar to the perturbation step used in CNC-PGD (Algorithm~\ref{alg:CNC-PGD}). Although this may not be very common in practice, adaptive stepsizes are not unusual in the literature (see e.g. \cite{goyal2017accurate}).

\begin{algorithm}[h!]
\begin{algorithmic}[1]
   \STATE \textbf{Input:} $\tr$, $r$, $\eta$, and $T$\;\;\;\; ($\eta < r$)
   \FOR{$t=1,2,\dots, T$} 
        \IF{$(t \mod \tr) = 0$}
            \STATE $\pw_t \leftarrow \w_t$  \qquad\qquad\qquad\quad\textit{\footnotesize{\# used in the analysis}}\\
            \STATE    $\w_{t+1} \leftarrow \w_t - r \gf_\z(\w_t)$ \; \textit{\footnotesize{\#} $z \stackrel{\text{i.i.d}}{\sim} \P$}
        \ELSE
        \STATE $\w_{t+1} \leftarrow \w_{t} - \eta \gf_\z(\w_t) $ \; \textit{\footnotesize{\#} $z \stackrel{\text{i.i.d}}{\sim} \P$}
        \ENDIF
   \ENDFOR
   \STATE \textbf{return} $\Tilde{\w}_t$ uniformly from $\{\Tilde{\w}_t| t< T\}$. 
\end{algorithmic}
\caption{CNC-SGD}
\label{alg:cnc_sgd}
\end{algorithm}

\paragraph{Parameters} The analysis of CNC-SGD relies on the particular choice of parameters presented in Table~\ref{tab:sgd_parameters}.

\begin{table}[h!]
\footnotesize
    \centering
    \begin{tabular}{c c c}
    Parameter  & Value & Dependency to $\epsilon$ 
    \\
    \hline
     $r$ & 
     $c_1\delta \gamma \epsilon^2 /(\ell^3 L)$ & $\bigo(\epsilon^{2})$
     \\
    $\eta$ &
    $c_2 \gamma^2 \delta^2 \epsilon^5/(\ell^6 L^2)$  & $\bigo(\epsilon^{5})$ 
    \\
    $\tr$ & $c_3(\eta \epsilon )^{-1}\log(\ell L/(\eta \epsilon r ))$ & $\bigo(\epsilon^{-6} \log^2(1/\epsilon))$
    \\ 
    $T$ &  $2 \tr (f(\w_0) - f^*)/(\delta \ft)$  & $\bigo(\epsilon^{-10} \log^2(1/\epsilon))$  \\ 
    \hline \\
\end{tabular}
    \caption{Parameters of CNC-SGD: the constants $c_1,c_2$, and $c_3$ are independent of the parameters $\gamma$,$\delta$, $\epsilon$, $\rho$, and $L$ (see Appendix B for more details).}
    \label{tab:sgd_parameters}
\end{table}
\begin{theorem} \label{theorem:sgd_convergence}
Let the stochastic gradients $\gf_\z(\w_t)$ in CNC-SGD satisfy Assumption \ref{ass:CNC} and let $f$, $f_\z$ satisfy Assumption \ref{ass:smoothness}. Then Algorithm \ref{alg:cnc_sgd} returns an $\left(\epsilon,\sqrt{\rho}\epsilon\right)$-second-order stationary point with probability at least $(1-\delta)$ after $$\bigo\left( \left( \frac{L^3 \ell^{10} }{\delta^4 \gamma^4 } \right) (\epsilon^{-10}) \log^2\left(\frac{\ell L}{\epsilon\delta \gamma} \right)\right)$$ steps, where $\delta<1$. 
\end{theorem}

\paragraph{Remarks}
As reported in Table~\ref{table:dim_dep}, perturbed SGD - with isotropic noise - converges to an $(\epsilon,\epsilon^{1/4})$-second-order stationary point in $\bigo(d^p \epsilon^{-4})$ steps~\cite{ge2015escaping}. Here, we prove that under the CNC assumption, vanilla SGD - i.e. without perturbations - converges to an  $(\epsilon,\sqrt{\rho}\epsilon^{2/5})$second-order stationary point using $\tbigo(\epsilon^{-4})$ stochastic gradient steps. 
Although our result is worse in terms of the first-order optimality, it yields an improvement by an $\epsilon^{0.15}$-factor in terms of second-order optimality (note that the focus of this paper is the second-order optimality). However, this second-order optimality rate is still worse by an $\epsilon^{-0.1}$-factor compared to the best known convergence rate for perturbed SGD established by~\cite{Zhang2017AHT}, which requires $\bigo(d^p \epsilon^{-4})$ iterations for an $(\epsilon,\epsilon^{1/2})$-second-order stationary point.
One can even improve the convergence guarantee of SGD by using the NEON framework~\cite{allen2017neon2,xu2017first} but a perturbation with isotropic noise is still required. The theoretical guarantees we provide in Theorem~\ref{theorem:sgd_convergence}, however, are based on a less restrictive assumption. As we prove in the following Section, this assumption actually holds for stochastic gradients when learning half-spaces. Subsequently, in Section \ref{sec:EXP}, we present empirical observations that suggest its validity even for training wide and deep neural networks.

%% file: 05_ncc_sgd.tex

\section{Learning Half-spaces with Correlated Negative Curvature}\label{sec:LHS} 

The analysis presented in the previous sections relies on the CNC assumption introduced in Eq.~(\ref{eq:CNC}). As mentioned before, this assumption is weaker than the isotropic noise condition required in previous work. In this Section we confirm the validity of this condition for the problem of learning half-spaces which is a core problem in machine learning, commonly encountered when training Perceptrons, Support Vector Machines or Neural Networks~\cite{zhang2015learninghs}. Learning a half-space reduces to a minimization problem of the following form 
\begin{align}
\min_{\w\in\mathbb{R}^d} \left[f(\w) := \E_{\z \sim \P} \left[ \varphi(\w^\top \z) \right]\right],
\label{eq:half_space}
\end{align} 
where $\varphi$ is an arbitrary loss function and the data distribution $\P$ might have a finite or infinite support. There are different choices for the loss function $\varphi$, e.g. zero-one loss, sigmoid loss or piece-wise linear loss \cite{zhang2015learninghs}. Here, we assume that $\varphi(\cdot)$ is differentiable. Generally, the objective $f(\w)$ is non-convex and might exhibit many local minima and saddle points.
Note that the stochastic gradient is unbiased and defined as 
\begin{align} 
\gf_{\z}(\w) = \varphi'(\w^\top \z) \z, \quad \gf(\w) = \E_{\z} \left[ \gf_{\z}(\w) \right],
\label{eq:half_space_gradient}
\end{align}
where the samples $\z$ are drawn from the distribution $\P$.

\begin{figure*}[t!]
	\begin{center}
          \begin{tabular}{@{}c@{\hspace{2mm}}c@{\hspace{2mm}}c@{\hspace{2mm}}c@{}}
          \vspace{-5.5pt}
            \includegraphics[width=0.33\linewidth]{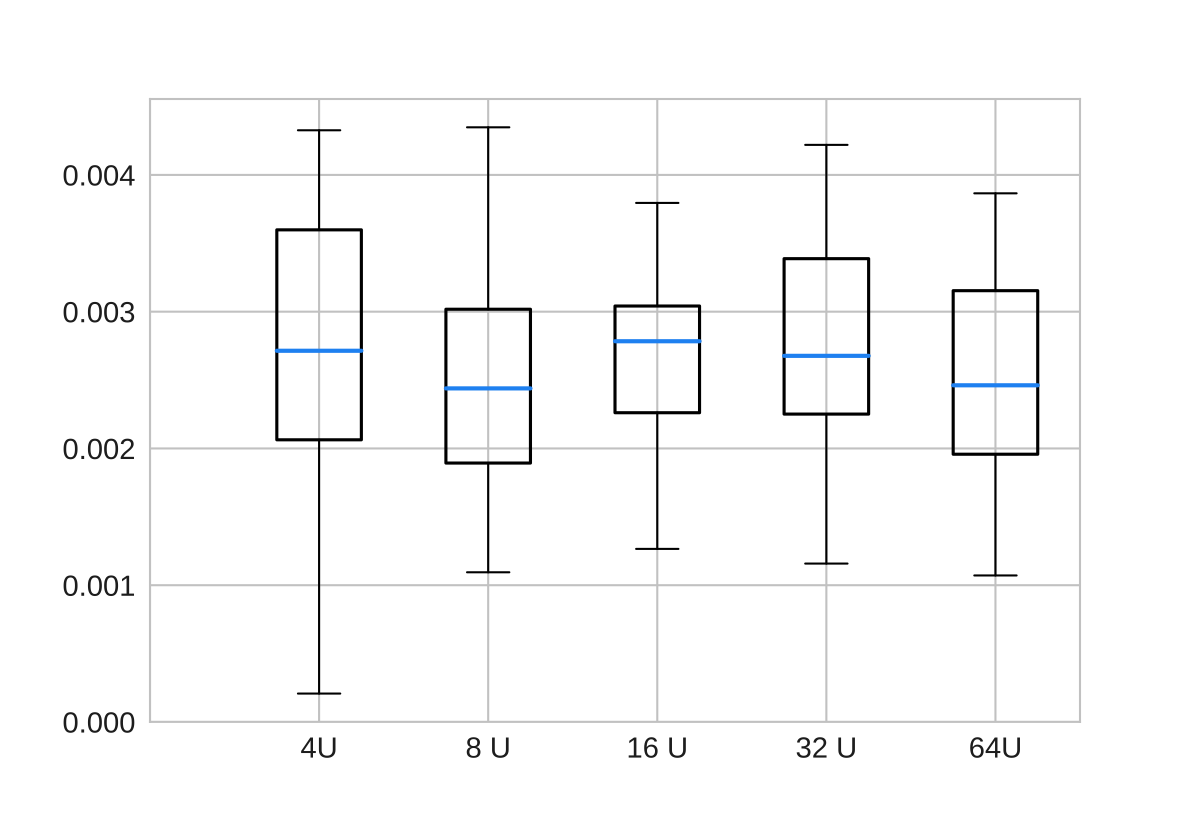} &
            \includegraphics[width=0.33\linewidth]{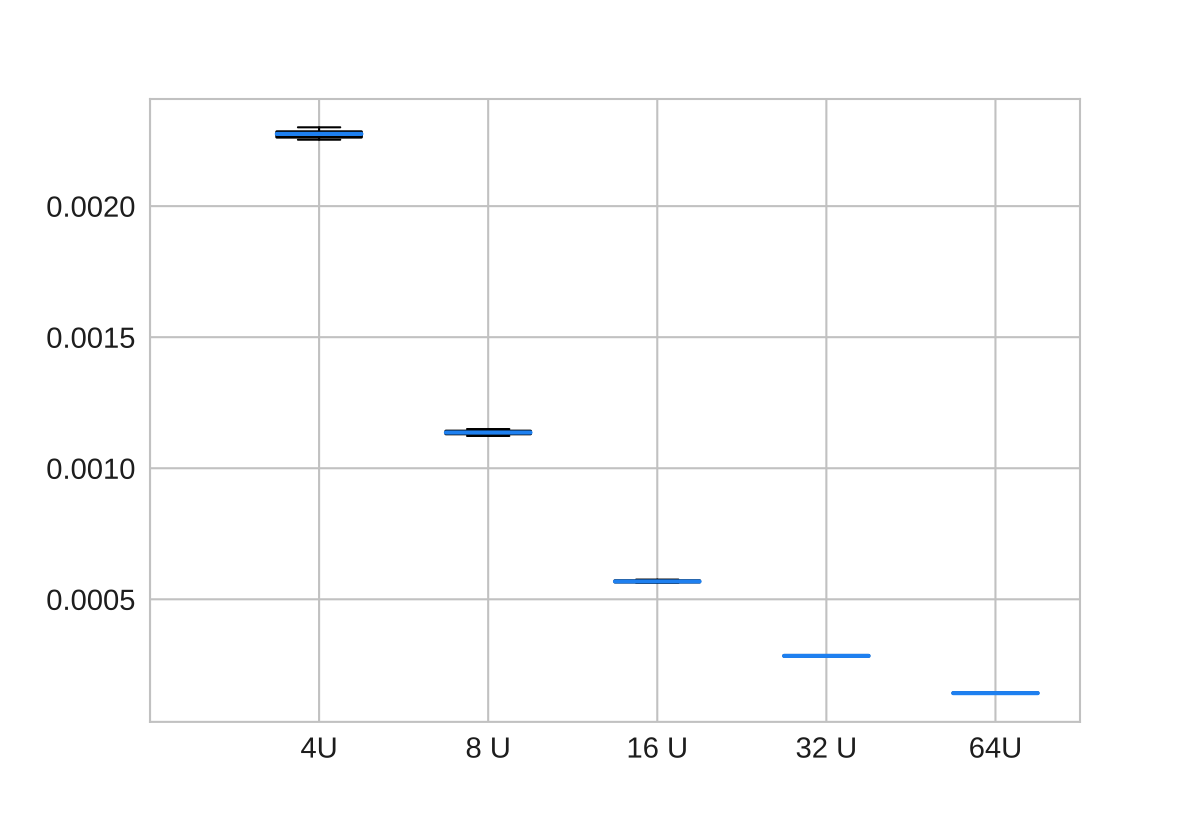} & \includegraphics[width=0.33\linewidth]{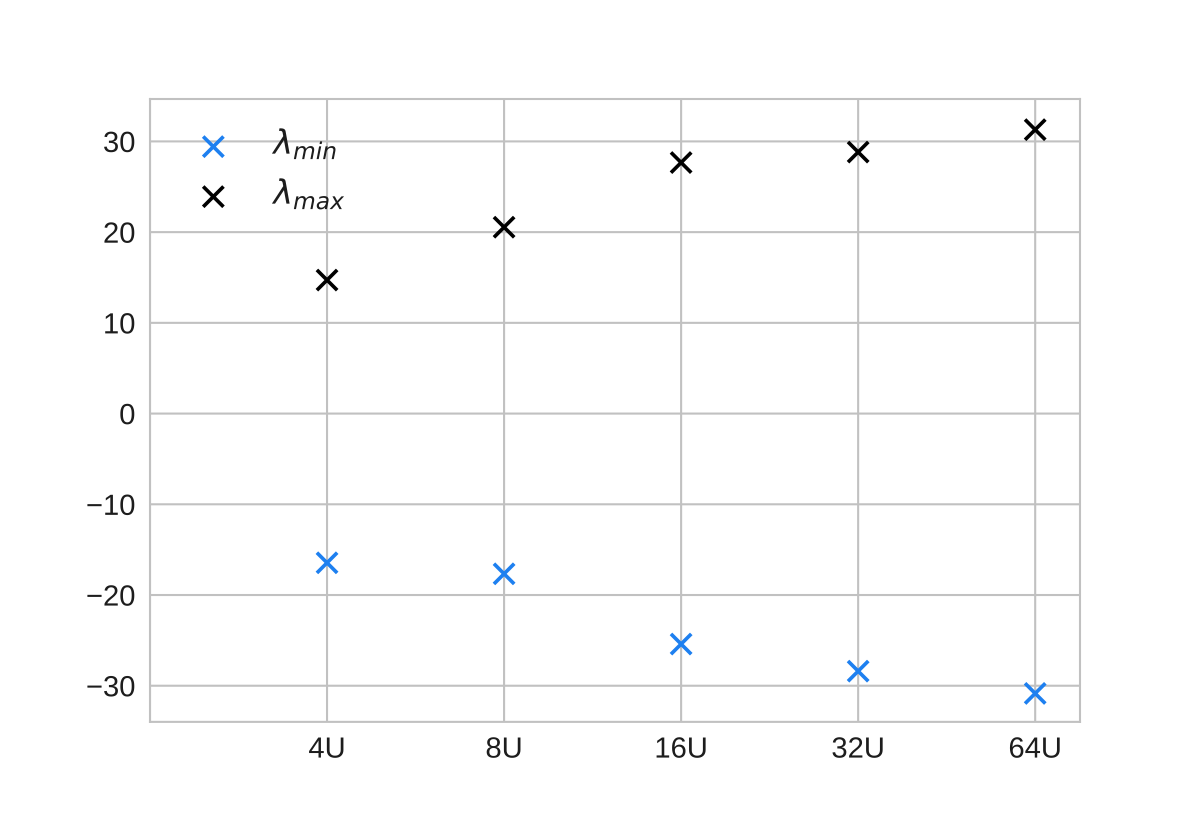} \\  \vspace{-1pt}
            
            \includegraphics[width=0.33\linewidth]{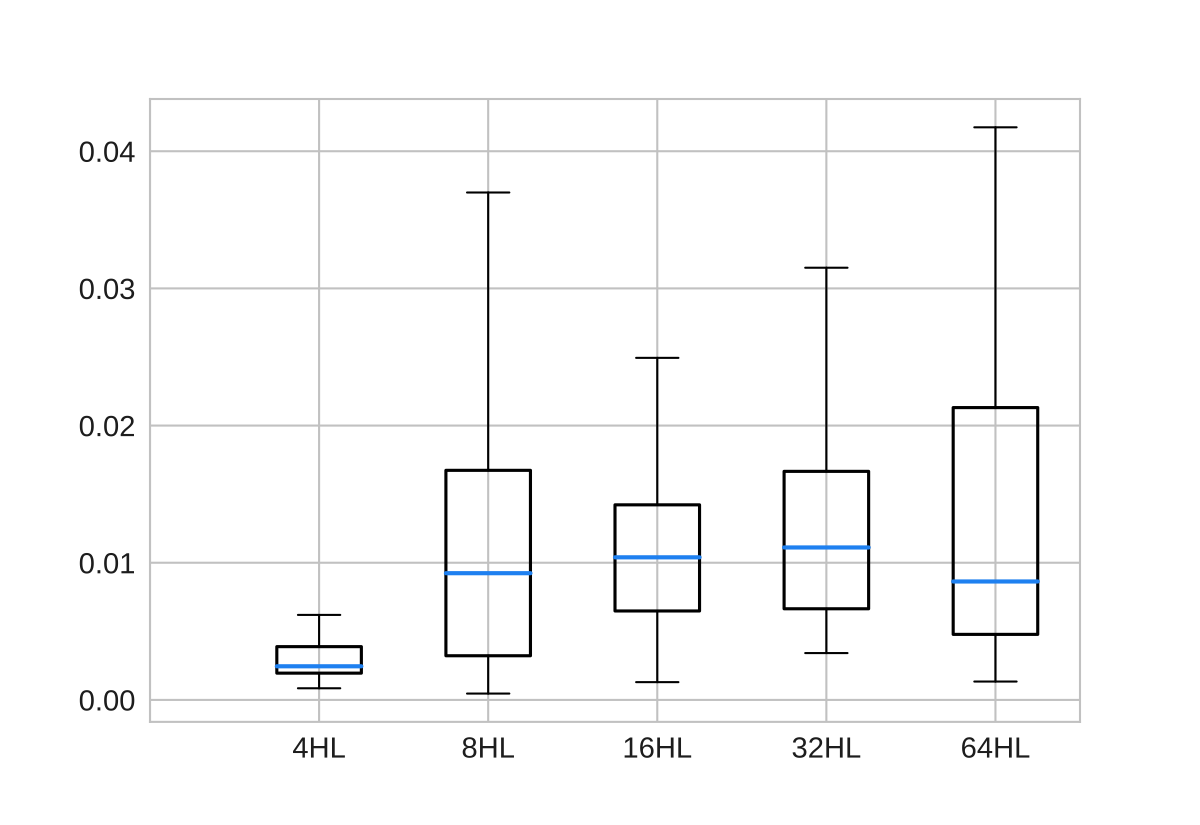} &
            \includegraphics[width=0.33\linewidth]{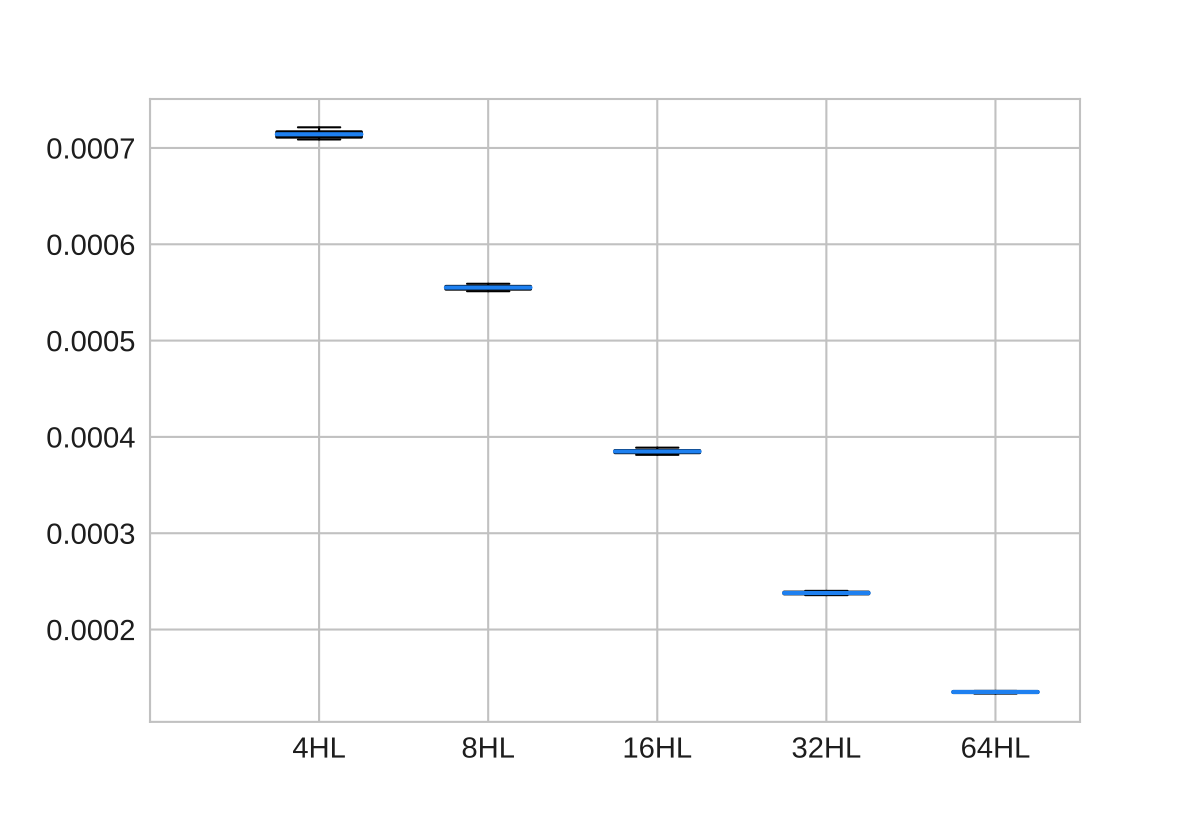} & \includegraphics[width=0.33\linewidth]{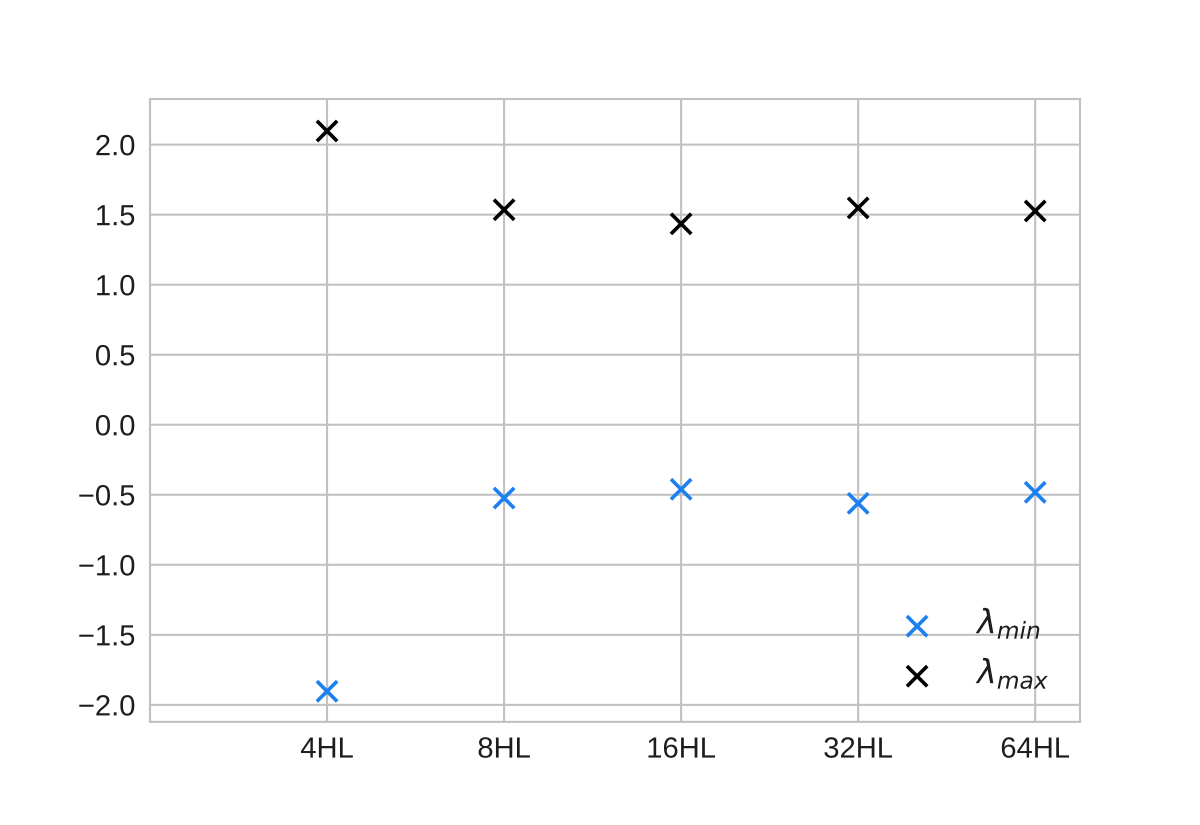} \\  \vspace{-1pt}
         	1. \footnotesize{{ Stochastic gradients}} &
            2. \footnotesize{{ Isotropic noise }} &
            3. \footnotesize{{Extreme eigenvalues}}
	  \end{tabular}
          \caption{Average variance of stochastic gradients (1) and isotropic noise (2) along eigenvectors corresponding to $\lambda_{min}$ and extreme eigenvalues (3) of 30 random weight settings in a 1-Layer Neural Network with increasing number of units U (top) and multi-layer Neural Network with increasing number of hidden layers HL (bottom).}
          \label{fig:cov_left}
	\end{center}
\end{figure*}

\paragraph{Noise isotropy vs. CNC assumption.}
First, one can easily find a scenario where the noise isotropy condition is violated for stochastic gradients.
Take for example the case where the data distribution from which $\z$ is sampled lives in a low-dimensional space $\mathcal{L} \subset \R^d$. In this case, one can prove that there exists a vector $\u \in \R^d$ orthogonal to all $\z \in \L$. Then clearly $\E\left[ (\u^\top \gf_{\z}(\w))^2 \right] =0 $ and thus $\gf_\z(\w)$ does not have components along all directions. 

However - under mild assumptions - we show that the stochastic gradients do have a significant component along directions of negative curvature. Lemma \ref{lemma:CNC_lowerbound} makes this argument precise by establishing a lower bound on the second moment of the stochastic gradients projected onto eigenvectors corresponding to negative eigenvalues of the Hessian matrix $\hf(\w)$. To establish this lower bound we require the following structural property of the loss function $\varphi$.

\begin{assumption}
Suppose that the magnitude of the second-order derivative of  $\varphi$ is bounded by a constant factor of its first-order derivative, i.e.
\begin{align} \label{eq:seim-self-concordant}
| \varphi''(\alpha)| \leq c | \varphi'(\alpha)| 
\end{align}
holds for all $\alpha$ in the domain of $\varphi$ and $c>0$.
\label{eass:seim-self-concordant}
\end{assumption}

The reader might notice that this condition resembles the self-concordant assumption often used in the optimization literature~\cite{nesterov2013introductory}, for which the second derivative is bounded by the third derivative. One can easily check that this condition is fulfilled by commonly used activation functions in neural networks, such as the sigmoid and softplus. We now leverage this property to prove that the stochastic gradient $\gf_\z(\w)$ satisfies Assumption~\ref{ass:CNC} (CNC).

\begin{lemma}\label{lemma:CNC_lowerbound}
Consider the problem of learning half-spaces as stated in Eq. (\ref{eq:half_space}), where $\varphi$ satisfies Assumption~\ref{eass:seim-self-concordant}. Furthermore, assume that the support of $\P$ is a subset of the unit sphere.\footnote{This assumption is equivalent to assuming the random variable $\z$ lies inside the unit sphere, which is common in learning half-space~\cite{zhang2015learninghs}.} Let $\v$ be a unit length eigenvector of $\hf(\w)$ with corresponding eigenvalue $\lambda<0$. Then 
\begin{align} 
\E_\z \left[ (\gf_{\z}(\w)^\top\v)^2 \right] \geq  (\lambda/c)^2. 
\end{align} 
\end{lemma} 

\paragraph{Discussion} Since the result of Lemma \ref{lemma:CNC_lowerbound} holds for any eigenvector $\v$ associated with a negative eigenvalue $\lambda<0$, this naturally includes the eigenvector(s) corresponding to $\lambda_{min}$. As a result, Assumption \ref{ass:CNC} (CNC) holds for stochastic gradients on learning half-spaces. Combining this result with the derived convergence guarantees in Theorem~\ref{theorem:pgd} implies that a mix of SGD and GD steps (Algorithm \ref{alg:CNC-PGD}) obtains a second-order stationary point in polynomial time. Furthermore, according to Theorem~\ref{theorem:sgd_convergence}, vanilla SGD obtains a second-order stationary point in polynomial time without \textit{any} explicit perturbation. Notably, both established convergence guarantees are dimension free. 

Furthermore, Lemma \ref{lemma:CNC_lowerbound} reveals an interesting relationship between stochastic gradients and eigenvectors at a certain iterate $\w$. Namely, the variance of stochastic gradients along these vectors scales proportional to the magnitude of the negative eigenvalues within the spectrum of the Hessian matrix. This is in clear contrast to the case of isotropic noise variance which is \textit{uniformly} distributed along all eigenvectors of the Hessian matrix. The difference can be important form a generalization point of view. Consider the simplified setting where $\varphi$ is square loss. Then the eigenvectors with large eigenvalues correspond to the principal directions of the data. In this regard, having a lower variance along the non-principal directions avoids over-fitting.

In the following section we confirm the above results and furthermore show experiments on Neural Networks that suggest the validity of these results beyond the setting of learning half-spaces.

%% file: 06_experiments.tex

\section{Experiments}\label{sec:EXP}
In this Section we first show that vanilla SGD (Algorithm \ref{alg:cnc_sgd}) as well as GD with a stochastic gradient step as perturbation (Algorithm \ref{alg:CNC-PGD}) indeed escape saddle points. Towards this end, we initialize SGD,  GD, perturbed GD with isotropic noise (ISO-PGD) \cite{jin2017escape} and CNC-PGD close to a saddle point on a low dimensional learning-halfspaces problem with Gaussian input data and sigmoid loss. Figure \ref{fig:escaping_saddles} shows suboptimality over epochs for an average of 10 runs. The results are in line with our analysis since all stochastic methods quickly find a negative curvature direction to escape the saddle point. See Appendix E for more details.\footnote{Rather than an encompassing benchmark of the different methods, this result is to be seen as a proof of concept.}

\begin{figure}[h!]
\centering
\includegraphics[width=175pt, trim={20pt 15pt 20pt 20pt},clip]{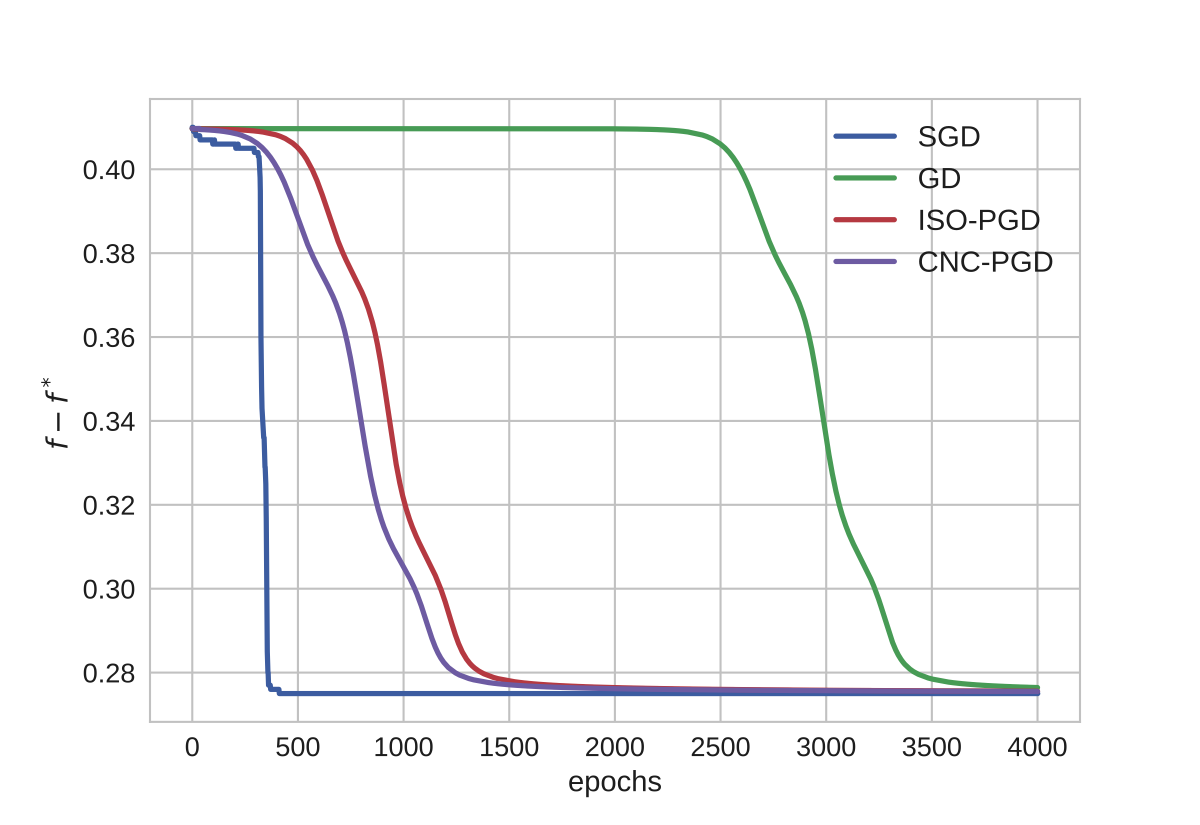}
\caption{Learning halfspaces ($n=40,d=4$): The stochastic methods need less iterations to escape the saddle.}
\label{fig:escaping_saddles}
\end{figure}

Secondly - and more importantly -  we study the properties of the variance of stochastic gradients depending on the width and depth of neural networks. All of these experiments are conducted using feed-forward networks on the well-known MNIST classification task ($n=70'000$). Specifically, we draw $m=30$ random parameters $\w_i$ in each of these networks and test Assumption \ref{ass:CNC} by estimating the second moment of the stochastic gradients projected onto the eigenvectors $\v_k$ of $\nabla^2 f(\w_i)$ as follows
\begin{equation}
\mu_k= \frac{1}{m} \sum_{i=1}^{m} \left( \frac{1}{n}\sum_{j=1}^n \left(\gf_j(\w_i)^\top \v_k\right)^2\right).
\end{equation}

We do the same for $n$ isotropic noise vectors drawn from the unit ball $\mathcal{B}^d$ around each $\w_i$.\footnote{For a fair comparison all involved vectors were normalized.} Figure \ref{fig:cov_left} shows this estimate for eigenvectors corresponding to the minimum eigenvalues for a 1 hidden layer network with increasing number of units (top) and for a 10 hidden unit network with increasing number of layers (bottom). Similar results on the entire negative eigenspectrum can be found in Appendix E. Figure \ref{fig:cov_over_ev} shows how $\mu_k$ varies with the magnitude of the corresponding negative eigenvalues $\lambda_k$. Again we evaluate 30 random parameter settings in neural networks with increasing depth. Two conclusions can be drawn from the results:
(i) Although the variance of isotropic noise along eigenvectors corresponding to $\lambda_{\min}$ decreases as $\bigo(1/d)$, the stochastic gradients maintain a significant component along the directions of most negative curvature independent of \textit{width} and \textit{depth} of the neural network (see Figure \ref{fig:cov_left}),
(ii) the stochastic gradients yield an increasing variance along eigenvectors corresponding to larger eigenvalues (see Figure \ref{fig:cov_over_ev}).
These findings suggest important implications. (i) justify the use and explain the success of training wide and deep neural networks with pure SGD despite the presence of saddle points. (ii) suggests that the bound established in Lemma~\ref{lemma:CNC_lowerbound} may well be extended to more general settings such as training neural networks and illustrates the implicit regularization of optimization methods that rely on stochastic gradients since directions of large curvature correspond to principal (more robust) components of the data for many machine learning models.
\begin{figure}
\centering
\includegraphics[width=175pt]{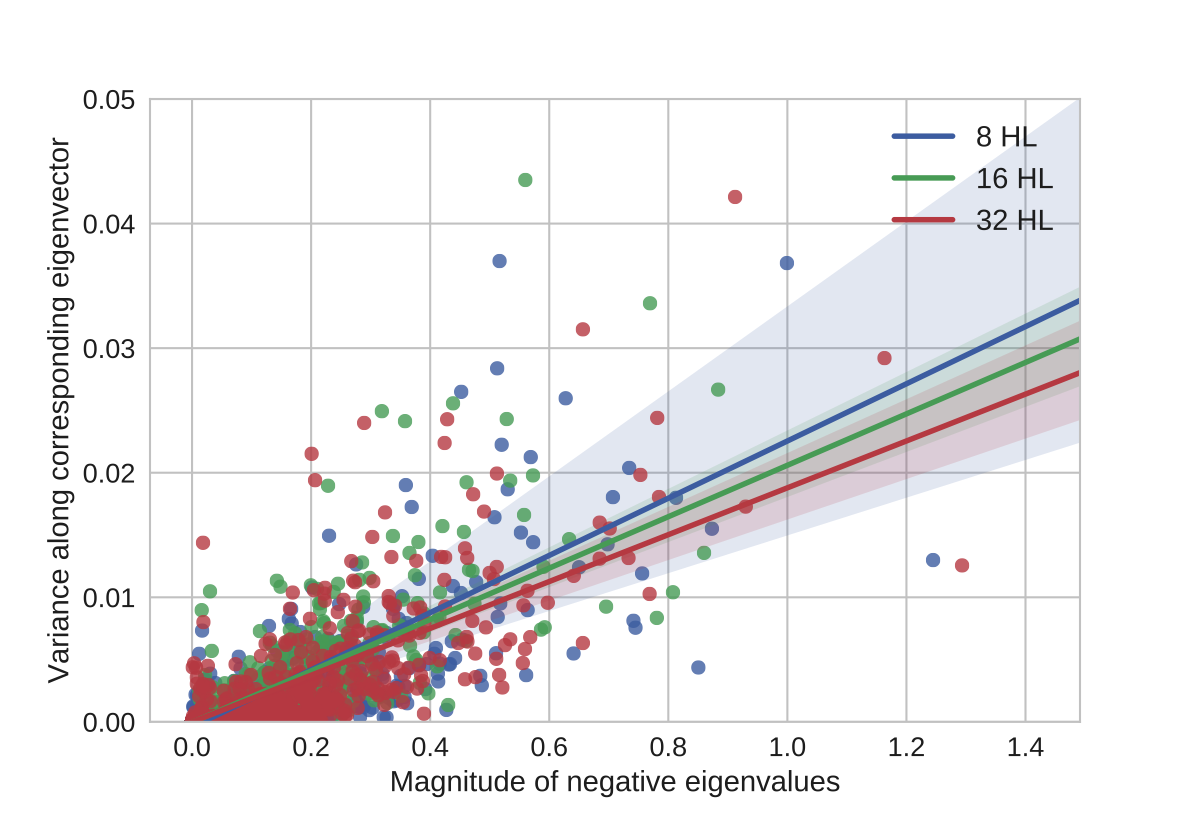}
\vspace{-5mm}
\caption{Variance of stochastic gradients along eigenvectors corresponding to eigenvalues of different magnitudes computed on neural networks with 8, 16 and 32 hidden layers. Scatterplot and fitted linear model with 95\% confidence interval.}
\label{fig:cov_over_ev}
\end{figure}



%% file: 07_conclusion.tex

\section{Conclusion}

In this work we have analyzed the convergence of PGD and SGD for optimizing non-convex functions under a new assumption -named CNC - that requires the stochastic noise to exhibit a certain amount of variance along the directions of most negative curvature. This is a less restrictive assumption than the noise isotropy condition required by previous work which causes a dependency to the problem dimensionality in the convergence rate. We have shown theoretically that stochastic gradients satisfy the CNC assumption and reveal a variance proportional to the eigenvalue's magnitude for the problem of learning half-spaces. Furthermore, we provided empirical evidence that suggests the validity of this assumption in the context of neural networks and thus contributes to a better understanding of training these models with stochastic gradients. Proving this observation theoretically and investigating its implications on the optimization and generalization properties of stochastic gradients methods is an interesting direction of future research.

%% file: 08_appendix.tex
\onecolumn

\appendix
\part*{Appendix}
\section{Preliminaries}
\paragraph{Assumptions}

Recall that we assumed the function $f$ is L-smooth (or L-Lipschitz gradient) and $\rho$-Lipschitz Hessian. We define these two properties below.

\begin{definition}[Smooth function]
A differentiable function $f$ is L-smooth (or L-Lipschitz gradient) if
\begin{equation}
    \| \gf(\w_1) - \gf(\w_2) \| \leq L \| \w_1 - \w_2 \|, \quad\quad \forall\: \w_1, \w_2 \in \R^d
\end{equation}
\end{definition}

\begin{definition}[Hessian Lipschitz]
A twice-differentiable function $f$ is $\rho$-Lipschitz Hessian if
\begin{equation} \label{eq:Hessian_Lipschitzness}
    \| \hf(\w_1) - \hf(\w_2) \| \leq \rho \| \w_1 - \w_2 \|, \quad\quad \forall\: \w_1, \w_2 \in \R^d
\end{equation}
\end{definition}

\begin{definition}[Bounded Gradient]
A differentiable function $f$ (of form of \eqref{eq:f_x}) is $\ell$-bounded gradient \footnote{This assumption guarantees $\ell$-Lipschitzness of $f$.} if
\begin{equation}
    \| \gf_\z(\w) \| \leq \ell, \quad\quad \forall\: \w \in \R^d
\end{equation}
\end{definition}

\begin{definition}[Stochastic Gradient Lipschitz] A differentiable function (of form of \eqref{eq:f_x}) has $\beta$-Lipschitz stochastic gradients if 
\begin{align} \label{eq:sg_lipschitz}
\| \nabla_\z f(\w_1) - \nabla_\z f(\w_2) \| \leq \beta \| \w_1 - \w_2 \|, \quad\quad \forall \w_1,\w_2, \z \in \R^d
\end{align} 

\end{definition}

\paragraph{Convergence of SGD on a smooth function}

\begin{framed}
\begin{lemma} \label{lemma:decrease_f_sgd} 
Let $\w_{t+1}$ be obtained from one stochastic gradient step at $\w_t$ on the $L$-smooth objective $f$, namely  
\[ 
\w_{t+1} = \w_t - \eta \gf_\z(\w_t)
\]
where $\E_\z \left[ \gf_\z(\w_t) \right] = \gf(\w_t)$ and $f_\z$ is $\ell$-bounded gradient. Then the function value decreases in expectation as
\begin{align} 
\E_\z \left[ f(\w_{t+1}) \right] - f(\w_t)  \leq - \eta \E \| \gf(\w_t) \|^2 + L \eta^2 \ell^2/2.
\end{align}
\end{lemma}
\end{framed}

\begin{proof}
 The proof is based on a straightforward application of smoothness:
\begin{align*} 
\E_{\z} \left[ f(\w_{t+1}) \right] - f(\w_t)  & \leq -\eta (\gf(\w_t) )^\top\E \left[ \gf_\z(\w_t) \right]  + L/2 \eta^2 \E \|  \gf_\z(\w_t)\|^2 \\ 
& \leq - \eta \| \gf(\w_t) \|^2 + L \eta^2\| \gf_\z(\w_t)\|^2 /2 \\ 
& \leq - \eta \| \gf(\w_t) \|^2 + L \eta^2 \ell^2/2.
\end{align*}

\end{proof}

\paragraph{Bounded series}
\begin{lemma} \label{lemma:auxiliary}
For all $1>\beta>0$, the following series are bounded as
\begin{align} 
\sum_{i=1}^t (1+\beta)^{t-i} & \leq 2\beta^{-1} (1+\beta)^t   \\
\sum_{i=1}^t (1+\beta)^{t-i} i & \leq 2\beta^{-2}(1 + \beta)^{t}\\
\sum_{i=1}^t (1+\beta)^{t-i} i^2 & \leq 6\beta^{-3}(1 + \beta)^{t}
\end{align}
\end{lemma}


\begin{proof}
The proof is based on the following bounds on power series for $|z|<1$: 
\begin{align*}
\sum_{k=1}^\infty z^k &\leq 1/(1-z) \\ 
\sum_{k=1}^\infty z^{k}k &= z/(1-z)^2\\
\sum_{k=1}^\infty z^{k}k^2 &= z(1+z)/(1-z)^3.
\end{align*} 
Yet, for the sake of brevity, we omit the subsequent (straightforward) derivations needed to prove the statement.
\end{proof}


\section{PGD analysis}

 \subsection{Choosing the parameters}
Table~\ref{tab:pgd_parameters_constraints} represents the choice of parameters together with the collection of required constraints on the parameters.
This table summarizes our approach for choosing the parameters of CNC-PGD presented in Algorithm~\ref{alg:CNC-PGD}.
\begin{table}[h!]
\footnotesize
    \centering
    \begin{tabular}{c c c c c c}
    Parameter  & Value & Dependency to
    $\epsilon$ & Constraint & Source & constant \\
    \hline
    $\eta$ & $1/L$ & Independent & $\eta\leq1/L$ & Lemma~\ref{lemma:decrease_f}\\ 
    $r$ & $c_1(\delta \gamma \epsilon^{4/5})/(\ell^3 L^2)$ & $\bigo(\epsilon^{4/5})$ & $\gamma  \epsilon^{4/5}/(16  L\ell^3)$ &  Lemma~\ref{lemma:small_grad_large_nc_restated} (Eq.~\eqref{eq:pgd_r_constraint}) & $c_1 = 1/64$\\ 
    $\tr$ & $c_2 L(\sqrt{\rho} \epsilon^{2/5})^{-1}  \log(\ell L/(\gamma\delta\epsilon))$ & $\bigo(\epsilon^{-2/5} \log(1/\epsilon))$ & $c L(\sqrt{\rho} \epsilon^{2/5})^{-1} \log(\ell L/(\gamma r)))$ & Lemma~\ref{lemma:small_grad_large_nc_restated} (Eq.~\eqref{eq:pgd_tr_constraint}) & $c_2 = c$\\ 
     $\ft$ & $c_3 \delta \gamma^2  \epsilon^{8/5}/(\ell^2 L)^2$ & $\bigo(\epsilon^{8/5})$ & $\leq \gamma  \epsilon^{4/5}r/(32 \ell) $ & Lemma~\ref{lemma:small_grad_large_nc_restated} (Eq.~\eqref{eq:pgd_ft_constraint}) & $c_3 = (64)^{-2}$
     \\ 
     $\ft$ & '' & '' & $\geq 2 L^2 (\ell r)^2 /\delta$ & Lemma~\ref{lemma:small_grad_small_nc_restrated}   (Eq.~\eqref{eq:pgd_ft_lowerbound}) \\
    $\gt$ & $\ft/\tr$ & $\bigo(\epsilon^2/\log(1/\epsilon))$ \\
    $T$ & $4 (f(\w_0) - f^*)/(\eta \delta \gt)$ & $\bigo(\epsilon^{-2}\log(1/\epsilon))$ \\ 
    \hline
\end{tabular}
    \caption{\textit{Parameters of CNC-PGD.}(Restated Table~\ref{tab:pgd_parameters})}
    \label{tab:pgd_parameters_constraints}
\end{table}
\subsection{Sharp negative curvature regime}
\begin{framed}
\begin{lemma} [Restated Lemma~\ref{lemma:small_grad_large_nc}]\label{lemma:small_grad_large_nc_restated}
 Let Assumption~\ref{ass:CNC} and~\ref{ass:smoothness} hold. Consider perturbed gradient steps (Algorithm \ref{alg:CNC-PGD} with parameters as in Table~\ref{tab:pgd_parameters}) starting from ${\pw}_t$ such that $\|\gf(\pw_t)\|^2 \leq \gt$. Assume the Hessian matrix $\hf(\pw_t)$ has a large negative eigenvalue, i.e. 
\begin{align} \label{eq:lambda_min_bound}
\lambda_{\min} (\hf(\pw_t)) \leq - \sqrt{\rho}\epsilon^{2/5}.
\end{align}
Then, after $\tr$ iterations the function value decreases as 
\begin{align} 
\E \left[ f(\w_{t+\tr}) \right] - f(\pw_t)  \leq - \ft,
\end{align}  where the expectation is over the sequence $\{\w_k\}_{t+1}^{t+\tr}$.  
 \end{lemma} 
 \end{framed}

\paragraph{Notation}
Without loss of generality, we assume that $t=0$. Let $\v$ be the eigenvector  And we use the simplified notation $\xi := \gf_\z(\pw_0)$, $\v := \v_0$. We also use the compact notations: 
\begin{align*} 
f_t := f(\w_t), \gf_t := \gf(\w_t),   \tilde{f} := f(\pw), \gft := \gf(\pw_t), \H := \hf(\pw), \gg_t := g(\w_t), 
\end{align*} 
Note that $\pw$ denote parameter $\w_0$ before perturbation and $\w_i$ is obtained by $i$ GD steps after perturbation. 
Recall the compact notation $\lambda$ as 
\begin{align*} 
\lambda := |\min\{ \lambda_{\min} \left(\hf(\pw), 0\}\right)|
\end{align*} 
Finally, we set $\kappa := 1+\eta \lambda$.

\paragraph{Proof sketch} The proof presented below proceeds by contradiction and is inspired by the analysis of accelerated gradient descent in non-convex settings as done in \cite{jin2017accelerated}.
We first assume that the sufficient decrease condition is not met and show that this implies an upper bound on the distance moved over a given number of iterations. We then derive a lower bound on the iterate distance and show that - for the specific choice of parameters introduced earlier - this lower bound contradicts the upper bound for a large enough number of steps $T$. We therefore conclude that we get sufficient decrease for $t > T$.
 
\textit{Proof of Lemma \ref{lemma:small_grad_large_nc_restated}:} 
\paragraph{Part 1: Upper bounding the distance on the iterates in terms of function decrease.} We assume that PGD does not obtain the desired function decrease in $\tr$ iterations, i.e. 
\begin{align} \label{eq:contradiction_assumption}
\E \left[ f(\w_{\tr}) - f(\tilde{\w})  \right] > - \ft. 
\end{align} 
The above assumption implies the iterates $\w_{t}$ stay close to $\tilde{\w}$, for all $t\leq \tr$. We formalize this result in the following lemma.
\begin{framed}
\begin{lemma}[Distance Bound]\label{lemma:expected_distance_bound}
Under the setting of Lemma \ref{lemma:small_grad_large_nc_restated}, assume Eq~\eqref{eq:contradiction_assumption} holds. Then the expected distance to the initial parameter can be bounded as 
\begin{align}
 \E \left[ \| \w_{t} - \tilde{\w} \|^2 \right] \leq2  \left(2\eta \ft  + \eta L(\ell r)^2  \right) t + 2 (\ell r)^2\quad\quad \forall t\leq\tr,
 \end{align} 
 as long as $\eta \leq 1/L$.
\end{lemma}
\end{framed}

\begin{proof}
Here, we use the proposed proof strategy of normalized gradient descent~\cite{levy2016power}. First of all, we bound the effect of the noise in the first step. Recall the first update of Algorithm \ref{alg:CNC-PGD} under the above setting
\begin{equation*}
\w_1= \pw-r\xi, \; \xi := \gf_\z(\pw).
\end{equation*}
Then by a straightforward application of lemma~\ref{lemma:decrease_f_sgd}, we have
\begin{align} \label{eq:first_step_bound}
\E \left[ f_1 - \tilde{f} \right] \leq -r\|\nabla \tilde{f}\|^2+\frac{L}{2}(\ell r)^2.
\end{align}
We proceed using the result of Lemma~\ref{lemma:decrease_f} that relates the function decrease to the norm of the visited gradients:
\begin{equation}
 \begin{aligned} 
 \E \left[ f_{\tr} - \tilde{f}\right] & = \sum_{t=1}^{\tr} \E \left[ f_t - f_{t-1} \right]  \\ 
 & \leq -\frac{\eta}{2}\sum_{t=1}^{\tr-1}   \E \| \gf_t \|^2  + \E \left[ f_1 - \tilde{f} \right]  \\ 
 & \leq  -\frac{\eta}{2}\sum_{t=1}^{\tr-1}  \E \| \gf_t \|^2 + \frac{L}{2}(\ell r)^2. \quad \text{\tiny{[Eq.~\ref{eq:first_step_bound}]}}
 \end{aligned} 
 \end{equation}

According to Eq.~\eqref{eq:contradiction_assumption}, the function value does not decrease too much. Plugging this bound into the above inequality yields an upper bound on the sum of the squared norm of the visited gradients, i.e.
 \begin{align} \label{eq:squared_norm_grads}
 \sum_{t=1}^{\tr-1} \E \| \gf_t \|^2 \leq (2 \ft+L(\ell r)^2)/\eta.
 \end{align} 
 Using the above result allows us to bound the expected distance in the parameter space as:
 \begin{equation}
 \begin{aligned} 
 \E \left[ \| \w_{t} - \w_1 \|^2 \right] & = \E \left[ \| \sum_{i=2}^{t} \w_i - \w_{i-1} \|^2\right]   \\ 
 & \leq \E \left[ \left( \sum_{i=2}^{t} \|  \w_i - \w_{i-1} \|\right)^2\right] \quad\quad \tiny{\text{[Triangle inequality]}} \\
 & \leq \E \left[ t \sum_{i=2}^{t} \| \w_i - \w_{i-1} \|^2 \right] \quad\quad \text{{\tiny [Cauchy-Schwarz inequality]}}\\ 
 & \leq t \left(\E \left[ \eta^2 \sum_{i=1}^{t-1}\| \gf_i \| ^2 \right]\right)   \\ 
 & \leq  \left(2\eta \ft  + \eta L(\ell r)^2  \right) t, \quad\quad  \forall t \leq \tr. \quad\quad \text{{\tiny [Eq. \eqref{eq:squared_norm_grads}]}}
 \end{aligned}
 \end{equation}

Replacing the above inequality into the following bound completes the proof: 
 \begin{equation*}
 \begin{aligned} 
 \E \| \w_t - \pw \|^2 & \leq 2 \E \| \w_t - \w_1 \|^2 + 2 \E \| \w_1 - \pw \|^2  \\ 
 & \leq 2  \left(2\eta \ft  + \eta L(\ell r)^2  \right) t + 2 (\ell r)^2
 \end{aligned}
 \end{equation*}

 \end{proof}
 
\paragraph{Part 2: Quadratic approximation} 

Since the parameter vector stays close to $\pw$ under the condition in Eq.~\eqref{eq:contradiction_assumption}, we can use a "stale" Taylor expansion approximation of the function $f$ at $\pw$:
\begin{align} 
g(\w) = \tilde{f} + (\w- \pw)^\top \gf(\pw) + \frac{1}{2} (\w- \pw)^\top \H (\w - \pw). \nonumber
\end{align} 
Using a stale Taylor approximation over all iterations is the essential part of the proof that is firstly proposed by \cite{ge2015escaping} for analysis of PSGD method. 
The next lemma proves that the gradient of $f$ can be approximated by the gradient of $g$ as long as $\w$ is close enough to $\pw$.
\begin{framed}
\begin{lemma}[Taylor expansion bound for the gradient \cite{nesterov2013introductory}] \label{lemma:quad_app}
For every twice differentiable, function $f:\mathbb{R}^d\rightarrow \mathbb{R}$ with $\rho$-Lipschitz Hessians the following bound holds true. 
\begin{align} 
\| \gf(\w) - \gg(\w) \| \leq \frac{\rho}{2} \| \w - \pw \|^2 
\end{align} 
\end{lemma}
\end{framed}
 
Furthermore, the guaranteed closeness to the initial parameter allows us to use the gradient of the quadratic objective $g$ in the GD steps as follows,
\begin{equation}
\begin{aligned} 
\w_{t+1} - \pw & = \w_{t}  -  \eta \gf_t - \pw  \\ 
& = \w_t -\pw-  \eta \gg_t +  \eta \left( \gg_t - \gf_t\right)   \\
& = (\Im - \eta \H) (\w_t - \pw) + \eta(\gg_t - \gf_t - \gf(\pw)) \\ 
& = \u_t + \eta (\bdelta_t + \d_t),
\label{eq:step_expansion}
\end{aligned}
\end{equation}

where the vectors $\u_t$, $\bdelta_t$ and $\d_t$ are defined in Table~\ref{tab:pgd_expansion_vectors}. 
\begin{table}[h!]
    \centering
    \begin{tabular}{c c c}
        Vector  & Formula & Indication  \\
        \hline
         $\u_t$ & $\left(\Im - \eta \H \right)^t(\w_1-\pw)$ & Power Iteration \\ 
         $\bdelta_t$ & $\sum_{i=1}^t \left( \Im - \eta \H \right)^{t-i} \left( \gf_t - \gg_t \right)$ & Stale Taylor Approximation Error \\ 
         $\d_t$ & $ -\sum_{i=1}^t \left( \Im - \eta \H \right)^{t-i} \gf(\pw)$ & Initial Gradient Dependency
    \end{tabular}
    \caption{Components of CNC-PGD expanded steps.}
    \label{tab:pgd_expansion_vectors}
\end{table}

As long as $\w_1 - \pw$ is correlated with the negative curvature, the norm of $\u_t$ grows exponentially. In this case, the upper bound of Lemma~\ref{lemma:expected_distance_bound} doesn't hold anymore after a certain number of iterations, as we formally prove in part 3. Indeed, the term $\u_t$ constitutes power iterations on the hessian matrix $\H$. The term $\bdelta_t$ arises from the stale Taylor approximation errors through all iterations. Assuming that $\w_t$ stays close to $\tilde{\w}$, we will bound this term. Finally, the $\d_t$ terms depend on the initial gradient. We will show that the distance $\E \|\w_1 - \pw \|^2$ is eventually dominated by the power iterates $\u_t$.
\paragraph{Part 3: Lower bounding the iterate distance.} 

\paragraph{A lower-bound on the distance}
Our goal is to provide a lower-bound on $\E \| \w_{\tr} - \w_0 \|^2$ that contradicts the result of Lemma~\ref{lemma:expected_distance_bound}. To obtain a lower bound on the distance, we use the classical result $\| a + b \|^2 \geq \| a \|^2 + 2 a^\top b$. Setting $a = \u_t$ and $b = \eta(\bdelta_t + \d_t)$ yields 
\begin{equation}
\begin{aligned} 
\E \| \w_{t+1} - \pw \|^2 & \geq \E \| \u_t \|^2 + 2 \eta \E \left[ \u_t^\top \bdelta_t \right] + 2 \eta \E \left[ \u_t^\top \right] \d_t \\
& \geq \E \| \u_t \|^2 - 2 \eta \E \left[ \| \u_t\| \| \bdelta_t \| \right] + 2 \eta \E \left[ \u_t^\top \right] \d_t  
\label{eq:distance_lower_bound_expansion_1}
\end{aligned}
\end{equation}

\paragraph{Removing the initial gradient dependency} The established lower-bound in Eq.~\eqref{eq:distance_lower_bound_expansion_1} has a dependency to the gradient $\gf(\pw)$ through the term $\E \left[ \u_t^\top\right] \d_t $. Intuitively, the initial gradient should not cause a problem for negative curvature exploration phase. More precisely, the third term of the lower bound of Eq.~\eqref{eq:distance_lower_bound_expansion_1} should be positive. This result is proven in the next lemma. 
\begin{framed}
\begin{lemma}[Removing initial gradient dependency] \label{lemma:removing_initial_gradient_dependency}
Under the setting of Lemma \ref{lemma:small_grad_large_nc_restated}, 
\begin{equation}
\E\left[ \u_t^\top \right] \d_t \geq 0.
\end{equation}
\end{lemma}
\end{framed}
\begin{proof}
Assumption \ref{ass:CNC} (CNC) implies that $\E \left[ \w_1 - \tilde{\w} \right] = -r\gf(\pw)$, hence the expectation of the power iteration term is 
\begin{align*} 
\E \left[ \u_t \right] = \left( \Im - \eta \H \right)^t \E \left[ \w_1 - \tilde{\w}\right] = -r\left( \Im - \eta \H \right)^t  \gf(\pw).
\end{align*}
Using this result, as well as the fact that $\left( \Im - \eta \H \right) \succeq 0 $ for $\eta \leq 1/L$ we have
\begin{align*}
    \E \left[ \u_t^\top \right] \d_t & =  r \left((\Im - \eta \H)^{t}\gf(\pw)\right)^\top \sum_{i=1}^t \left( \Im - \eta \H \right)^{t-i} \gf(\pw) \\
    & =  r\sum_{i=1}^t \gf(\pw)^\top(\Im - \eta \H)^{2t-i}\gf(\pw)\geq 0 ,
\end{align*}
which proves the assertion.
\end{proof}
Plugging the result of the last lemma into the lower-bound established in Eq.~\eqref{eq:distance_lower_bound_expansion_1} yields 
\begin{align} 
\E \left[ \| \w_t - \tilde{\w}\|^2  \right] \geq \E \| \u_t \|^2 - 2 \eta \E \left[ \| \u_t\| \| \bdelta_t \| \right]. \label{eq:distance_lower_bound_expansion_2}
\end{align}
 To complete our lower bound, we need : (I) a lower bound on $\E\| \u_t\|^2 $, (II) an upper bound on $\|\u_t\|$ and (III) an upper bound on $\E \| \bdelta_t \|$.

\paragraph{(I) Lower-bound on $\E \| \u_t \|^2  $}
Norm of $\u_t$ exponentially grows in iterations. Next lemma proves this claim. 
\begin{framed}
\begin{lemma}[Exponential Growing Power Iteration]
\label{lemma:lower_bound_u_2}
Under the setting of Lemma \ref{lemma:small_grad_large_nc_restated}, 
 $t$ steps of PGD yield an exponentially growing lower bound on the expected squared norm of $\u_t$, i.e.
\begin{align} 
\E \left[ \| \u_t \| ^2 \right] \geq \gamma r^2 \kappa^{2t}.
\end{align} 
\end{lemma} 
\end{framed}
\begin{proof} 
We first use Cauchy-Schwarz inequality to derive the following lower bound: 
\begin{equation}
\begin{aligned} \label{eq:taylor_lower_u}
\E \left[ \| \u_t \|^2 \right] & = \E \left[ \| \v \|^2 \| \u_t \|^2  \right]\\
& \geq \E \left[ ( \v^\top \u_t)^2 \right].
\end{aligned}
\end{equation}

Now, suppose $\H=U^\top \Sigma U$. Since any non-zero vector $\u\in\mathbb{R}^d$ is an eigenvector of the identity matrix we can decompose $I=U^\top IU$ and thus $(I-\eta \H)=U^\top (I-\eta \Sigma)U$. As a result, we have 
\begin{equation}\label{eq:v(I-H)}
\v^\top (I-\eta \H)=\v^\top(1-\eta \lambda_{\min}(\H))= \v^\top(1+\eta \lambda).   
\end{equation}
 for the leftmost eigenvector $\v$ of the Hessian $\H$. Plugging Equation (\ref{eq:v(I-H)}) into (\ref{eq:taylor_lower_u}) and recalling $\kappa := 1+\eta \lambda$ as well as the definition of $\u_t$ as given in Table \ref{tab:pgd_expansion_vectors} yields
\begin{align*}
\E \left[ \| \u_t \| ^2 \right]& \geq (1+\eta \lambda)^{2t}\E \left[ (\v^\top (\w_1 - \pw))^2 \right]  \\ 
& = r^2\kappa^{2t}\E \left[ (\v^\top \xi)^2 \right]  \\ 
& = \gamma r^2\kappa^{2t},
\end{align*} 
where the last inequality follows from Assumption \ref{ass:CNC}.
\end{proof}

\paragraph{(II) Upper-bound on $\| \u_t \|$} Using the definition of the scaling factor $r$ and the fact that the noise lies inside the unit sphere the next lemma proves a deterministic bound on this term. 
\begin{framed}
\begin{lemma} \label{lemma:upper_bound_u} 
Under the setting of Lemma \ref{lemma:small_grad_large_nc_restated}  the norm of $\u_t$ is \textbf{deterministically} bounded as
\begin{align} 
\| \u_t \| \leq \kappa^t \ell r .
\end{align} 
\end{lemma}
\end{framed}
\begin{proof}
Starting from the definition of $\u_t$ and recalling that $\xi$ has at most unit norm by Assumption \ref{ass:CNC}, we have
\begin{equation*}
\begin{aligned} 
\| \u_t \| & \leq \| \left(\Im - \eta \H\right)^t (\w_1 - \pw)\|  \\ 
& \leq \| \Im - \eta \H \|^t \| \w_1 - \pw \|  \\ 
& \leq  (1+\eta \lambda)^t r \| \xi \|  \\ 
& \leq \kappa^t r \ell
\end{aligned}
\end{equation*}

\end{proof}

\paragraph{(III) Upper bound on $\E \| \bdelta_t \|$}
To bound this term, we use the fact proved in  Lemma~\ref{lemma:expected_distance_bound} that $\w_t$ stays close to $\w_0$ for all $t\leq \tr$.  
\begin{framed}
\begin{lemma} \label{lemma:upper_bound_delta}
Under the setting of Lemma \ref{lemma:small_grad_large_nc_restated}, if $$\E \left[ f_{t+1} - \tilde{f} \right]  \geq - \ft,$$ then the norm of $\bdelta_t$ is bounded in expectation:
\begin{align}
\E \| \bdelta_t \| \leq \left( \frac{4\eta \ft  + 2\eta L(\ell r)^2}{(\eta \lambda)^2} + \frac{2(\ell r)^2}{\eta \lambda}\right) \rho \kappa^{t},\quad\quad \forall t\leq \tr.
\end{align} 
\end{lemma}
\end{framed}

\begin{proof}
Using the result of Lemma \ref{lemma:quad_app} as well as the distance bound established in Lemma~\ref{lemma:expected_distance_bound}, we have
\begin{equation}\label{eq:upper_bound_delta_1}
\begin{aligned}
    \E \left[ \| \bdelta_t \|  \right] & = \E \left[ \| \sum_{i=1}^t \left( \Im - \H \right)^{t-i} \left( \gg_i - \gf_i \right) \| \right]   \\ 
    & \leq \sum_{i=1}^t \| \Im - \eta \H \|^{t-i} \E \| \gg_i - \gf_i \|   \\ 
    & \leq \frac{\rho}{2}\sum_{i=1}^t \kappa^{t-i}  \E \|\w_i - \tilde{\w} \|^2  \quad\quad \text{\tiny{[]Lemma~\ref{lemma:quad_app}]}}\\
   & \leq \rho\sum_{i=1}^t \kappa^{t-i} \left( \left(2\eta \ft  + \eta L (\ell r)^2 \right) i + (\ell r)^2\right).  \quad\quad \text{\tiny{[Lemma~\ref{lemma:expected_distance_bound}]}}\\ 
   \end{aligned}
\end{equation}
Now, the results on convergence of power series derived in Lemma \ref{lemma:auxiliary} and the definition $\kappa:=1+\eta\lambda$ give
\begin{equation}\label{eq:convergence_series_on_kappa}
\sum_{i=1}^t \kappa^{t-i} \leq \frac{2\kappa^t}{\eta\lambda} \quad \text{and}\quad \sum_{i=1}^t \kappa^{t-i}i \leq \frac{ 2\kappa^t}{(\eta\lambda)^2}.
\end{equation}

By combining Equation (\ref{eq:upper_bound_delta_1}) and (\ref{eq:convergence_series_on_kappa}) we can establish the desired bound on  $\bdelta_t$:  

\begin{align}\label{eq:upper_bound_delta_2}
   \E \left[ \| \bdelta_t \|  \right]  \leq \left( \frac{4\eta \ft  + 2\eta L(\ell r)^2}{(\eta \lambda)^2} + \frac{2(\ell r)^2}{\eta \lambda}\right) \rho \kappa^{t}.
\end{align}

\end{proof}

We are now ready to combine the results of Lemma~\ref{lemma:lower_bound_u_2},~\ref{lemma:upper_bound_u}, and \ref{lemma:upper_bound_delta}, into Eq.~\eqref{eq:distance_lower_bound_expansion_2} in order to obtain the desired lower bound on the distance travelled by the iterates of PGD.
\begin{framed}
\begin{lemma}[Distance lower bound]
Under the setting of Lemma \ref{lemma:small_grad_large_nc_restated} and for each $t\leq \tr$ and for the choice of parameters as in Table \ref{tab:pgd_parameters_constraints} we have
\begin{equation}\label{eq:distance_lower_bound_final}
    \E \| \w_t - \pw \|^2 \geq \frac{ \gamma r^2 \kappa^{2t}}{4},
\end{equation}
where 
$\kappa:=1+\eta |\lambda_{\min} \left(\hf(\pw)\right)| $.
\end{lemma}
\end{framed}

\begin{proof}
To prove this statement we introduce each bound in Eq.~\eqref{eq:distance_lower_bound_expansion_2} step by step:
\begin{equation}
\begin{aligned} 
\E \| \w_t - \pw \|^2 & \geq \E \| \u_t \|^2 - 2 \eta \E \left[\| \u_t \| \| \bdelta_t \| \right] 
 \\ 
& \geq \gamma r^2 \kappa^{2t} - 2 \eta \E \left[\| \u_t \| \| \bdelta_t \| \right] \quad\quad \text{\tiny{[Lemma~\ref{lemma:lower_bound_u_2}]}} \\ 
& \geq \gamma r^2 \kappa^{2t} - 2 \eta \ell r \kappa^t \E \| \bdelta_t \| \quad\quad \text{\tiny{[Lemma \ref{lemma:upper_bound_u}]}} \\ 
& \geq \left(\gamma r - \frac{8 \rho \ell \ft  + 4 \rho L \ell^3 r^2 }{\lambda^2} - \frac{4 \rho \ell^3 r^2}{\lambda} \right) r \kappa^{2t}\label{eq:lower_bound_distance_final}  \quad\quad \text{\tiny{[Lemma~\ref{lemma:upper_bound_delta}]}}
\end{aligned} 
\end{equation}

We need the lower bound to be positive to complete the proof. In this regard, we require the following condition to hold,
\begin{align} \label{eq:reqs}
\gamma r - \underbrace{\frac{8\rho\ell \ft }{\lambda^2}}_{\leq \gamma r/4} - \underbrace{\frac{ 4L \rho \ell^3 r^2 }{\lambda^2}}_{\leq \gamma  r/4} - \underbrace{\frac{4\rho \ell^3 r^2}{\lambda}}_{\leq \gamma r/4} \stackrel{!}{\geq}  \frac{\gamma r}{4}.
\end{align}
Using the lower bound the absolute value of the minimum eigenvalue as $\lambda \geq \sqrt{\rho} \epsilon^{2/5}$ (in Eq.~\eqref{eq:lambda_min_bound}), we choose parameters $r$, $\ft$, and $\gt$ such that the above constraints are satisfied,\footnote{Note that the second requirement in (\ref{eq:reqs}) is always more restrictive than the last since $\lambda<L$.} i.e.
\begin{align} 
r & \leq \gamma  \epsilon^{4/5}/(16  L\ell^3)  \leq \gamma \lambda^2/(16\rho L\ell^3) \stackrel{[\lambda<L]}{\leq} (\gamma \lambda)/(16\rho\ell^3) \label{eq:pgd_r_constraint}\\
\ft & \leq \gamma  \epsilon^{4/5}r/(32\ell) \leq \gamma \lambda^2 r/(32\rho \ell)  \label{eq:pgd_ft_constraint}
\end{align}

These choices of parameters establish an exponential lower bound on the distance as 
\begin{equation*}
\E \| \w_t - \pw \|^2 \geq \frac{\gamma r^2 \kappa^{2t}}{4}.
\end{equation*}
\end{proof}

According to the result of Lemma~\ref{lemma:expected_distance_bound}, if the expected distance from the initial parameter is sufficiently large, then the assumption $\E \left[ f_t- \tilde{f}  \right] > - \ft $ cannot be valid. To derive the contradiction, we have to choose the number of step such that the established lower-bound exceeds the upper-bound in Lemma~\ref{lemma:expected_distance_bound}, namely 
\begin{align*} 
\frac{1}{4} \gamma r^2 \kappa^{2t} \stackrel{?}{\geq} 2  \left(2\eta \ft  + \eta L(\ell r)^2  \right) t + 2 (\ell r)^2.
\end{align*}
Since the left hand side is exponentially growing, we can derive the contradiction by choosing a large enough number of steps as: 
\begin{align} \label{eq:pgd_tr_constraint} 
\tr \geq c(\eta\lambda)^{-1} \log\left(\ell L /(\gamma r)\right) \geq c L(\sqrt{\rho} \epsilon^{2/5})^{-1} \log(\ell L/(\gamma r))),
\end{align} 
which completes the proof of Lemma \ref{lemma:small_grad_large_nc_restated}.
\begin{flushright}
$\square$
\end{flushright}

\subsection{Moderate negative curvature regime}
\begin{framed}
\begin{lemma} [Restate of Lemma~ \ref{lemma:small_grad_small_nc}] \label{lemma:small_grad_small_nc_restrated}
Let Assumption~\ref{ass:CNC} and~\ref{ass:smoothness} hold. Consider perturbed gradient steps (Algorithm \ref{alg:CNC-PGD} with parameters as in Table~\ref{tab:pgd_parameters_constraints}) starting from ${\pw}_t$ such that $\|\gf(\pw_t)\|^2 \leq \gt$. Then after $\tr$ iterations, the function value cannot increase by more than 
 \begin{align*} 
 \E \left[ f(\w_{t+\tr})\right]  - f(\pw_t) \leq \frac{ \eta\delta \ft}{4}, 
 \end{align*}
 where the expectation is over the sequence $\{\w_k\}_{t+1}^{t+\tr}$.
 \end{lemma} 
 \end{framed}

 \begin{proof} 
 Using the resulf of lemma~\ref{lemma:decrease_f_sgd}, we bound the decrease in the function value as 
 \begin{align} \label{eq:pgd_ft_lowerbound}
 \E \left[ f(\w_1) \right] - f(\pw_t) \leq  L(\ell r)^2/2  \leq \delta \eta \ft/4
 \end{align}
 Since there is no perturbation in following $\tr$ steps, GD doesn't increase the function value in following $\tr$-steps (according to the result of lemma~\ref{lemma:decrease_f}).
 \end{proof}
 

\section{SGD analysis }

\subsection{Parameters and Constraints}
 Table~\ref{tab:sgd_parameters_with_constraints} lists the parameters of CNC-SGD presented in Algorithm~\ref{alg:cnc_sgd} together with the constraints that determines our choice of parameters. These constraints are driven by the theoretical analysis. 
\begin{table}[h!]
\footnotesize
    \centering
    \begin{tabular}{c c c c c c}
    Parameter  & Value & Dependency to $\epsilon$ & Constraint & Constraint Origin & Constant
    \\
    \hline
     $r$ & 
     $c_1\delta \gamma \epsilon^2 /(\ell^3 L)$ & $\bigo(\epsilon^{2})$ &
     $\leq \gamma \lambda^2/(24 \ell^3 L \rho)$ &
     Lemma~\ref{lemma:sgd_large_nc} (Eq.~\eqref{eq:sgd_r_upperbound}) &
     $c_1 = 1/96$
     \\
     $r$ &
    " & " &
    $\leq \delta \epsilon^2/(2 \ell^2 L)$ &
    Eq.~\eqref{eq:sgd_large_gd_condition} &  \\
    $r$ & 
     " &
     " & 
     $ \leq \left( \delta \ft/(2\ell^2 L) \right)^{1/2}$ &
     Eq.~\eqref{eq:increase_bound_small_ss_sgd} 
     & 
     \\
     $\ft$ & 
     $c_2 \delta \gamma^2 \epsilon^4 /(\ell^4 L)$ &
     $\bigo(\epsilon^{4})$ & 
     $\leq \gamma\lambda^2 r/(48 \ell \rho)$ &
     Lemma~\ref{lemma:sgd_large_nc}(Eq.~\eqref{eq:sgd_ft_bound})
     & 
     \\
     $\ft$ & 
     '' &
    '' & 
     $ \leq \epsilon^2 r/2$ &
     Eq.~\eqref{eq:decrease_f_large_gd_sgd}
     & 
     $c_2 = c_1/48$\\ 
    $\eta$ &
    $c_4 \gamma^2 \delta^2 \epsilon^5/(\ell^6 L^2)$  & $\bigo(\epsilon^{5})$ &
    $ \leq c' \gamma  r \lambda^3 /(72 L \rho \ell^3)$ & 
    Lemma~\ref{lemma:sgd_large_nc}(Eq.\eqref{eq:sgd_eta_bound_2}) & $c_4 = c' c_1/72$\\
    $\eta$ &
    ''  & '' &
    $\leq r/\sqrt{\tr}$ & 
    \\
    $\omega$ 
    & $c_3\log(\ell L/(\eta \epsilon r ))$ & 
    $\bigo(\log(1/\epsilon))$ &  & & $c_3 = c$ (Eq.\eqref{eq:sgd_tr_bound}) \\
    $\tr$ & $(\eta \epsilon )^{-1}\omega$ & $\bigo(\epsilon^{-6} \log^2(1/\epsilon))$ & $= c (\eta \lambda )^{-1} \omega$ &
    Lemma~\ref{lemma:sgd_large_nc}(Eq.\eqref{eq:sgd_tr_bound}) \\ 
    $T$ &  $2 \tr (f(\w_0) - f^*)/(\delta \ft)$  & $\bigo(\epsilon^{-10} \log^2(1/\epsilon))$  \\ 
    \hline \\
\end{tabular}
    \caption{Parameters of CNC-SGD (Restated Table~\ref{tab:sgd_parameters}). By $\lambda^2 = \rho \epsilon^2$, one can reach the value of parameters from the established upperbounds. }
    \label{tab:sgd_parameters_with_constraints}
\end{table}
\subsection{Proof of the Main Theorem}

\begin{framed}
\begin{theorem} [Restated Theorem~\ref{theorem:sgd_convergence}]
Let the stochastic gradients $\gf_\z(\w_t)$ in CNC-SGD satisfy Assumption \ref{ass:CNC} and let $f$ and $f_\z$ satisfy Assumption \ref{ass:smoothness}. Then algorithm \ref{alg:cnc_sgd} returns an $\left(\epsilon,\sqrt{\rho}\epsilon\right)$-second order stationary point with probability at least $(1-\delta)$ after $$\bigo\left( \left( \frac{L^3 \ell^{10} }{\delta^4 \gamma^4 } \right) (\epsilon^{-10}) \log^2\left(\frac{\ell L}{\epsilon\delta \gamma} \right)\right)$$ steps, where $\delta<1$.
\end{theorem}
\end{framed}

\begin{proof}

We decompose the SGD step as
\begin{align} \label{eq:sgd_step_expansion}
\pw_{t} & = \w_{t-1} - r \xi_t \quad \text{\tiny{[Large Step-Size]}} \\
\w_{t+1} & = \w_t - \eta \gf(\w_t) + \eta \underbrace{(\gf(\w_t) - \gf_\z(\w_t))}_{\zeta_t}, \quad \text{\tiny{[Small Step-Size]}} 
\end{align} 
where the noise term $\zeta_t$s are i.i.d and zero-mean and the noise term $\xi_t$ satisfies CNC assumption~\ref{ass:CNC}. Our analysis relies on the CNC assumption only at steps with large step-size $r$. That is, we only exploit the negative curvature in the steps with a large step size. In this regard, we need to use the larger step size $r>\eta$ in these steps. This is different from Perturbed SGD -- with isotropic noise-- ~\cite{ge2015escaping} where the variance of perturbations in all steps is exploited in the analysis.  
\paragraph{Large gradient regime:}
 If the norm of the gradient is large, i.e. 
 \begin{align} \label{eq:sgd_large_gd_condition}
 \| \gf(\pw) \|^2 = \epsilon^2   \geq  2 \ell^2 L r/\delta, \quad \quad \text{[The choice of $r$ in Table~\ref{tab:sgd_parameters_with_constraints}]}
 \end{align}
 then the result on the one step convergence of SGD in Lemma~\ref{lemma:decrease_f_sgd} guarantees the desired decrease
 \begin{equation}
 \begin{aligned}
    \E \left[ f(\w_\tr) - f(\pw) \right] & = \sum_{i=0}^{{\tr}-1} \E \left[ \E  \left[ f(\w_{i+1}) - f(\w_{i}) | \w_i \right] \right] \\ 
    & = \E \left[ \E  \left[ f(\w_{1}) - f(\w_{0}) | \w_0 \right] \right] + \sum_{i=1}^{{\tr}-1} \E \left[ \E  \left[ f(\w_{i+1}) - f(\w_{i}) | \w_i \right] \right] \\
    & \stackrel{\text{lemma~\ref{lemma:decrease_f_sgd}}}{\leq} -r \| \nabla f(\pw) \|^2  - \sum_{i=1}^{{\tr}-1} r \E \| \gf_{\z_{r}} (\w_i) \|^2 + L r^2 l^2/2 + \tr L \eta^2 \ell^2/2 \\ 
    & \leq -r \| \nabla f(\pw) \|^2 + L \ell^2/2 \left( r^2 + \tr \eta^2 \right) \label{eq:function_value_decrease}
\end{aligned}
\end{equation}
The choice of $r$ and $\eta$ in Table~\ref{tab:sgd_parameters_with_constraints} ensures that 
\begin{align} \label{eq:r_eth_tr_sgd}
    \tr \eta^2 \leq r^2 
\end{align}
The above constraint simplifies the established upperbound on the function value decrease as  
\begin{equation} \label{eq:decrease_f_large_gd_sgd}
\begin{aligned}
 \E \left[ f(\w_{\tr}) - f(\pw) \right]  & \leq - r  \| \nabla f(\pw) \|^2 + L \ell^2 r^2 \\ 
 & \leq - \frac{1}{2} r \| \nabla f(\pw) \|^2 + r \left( L\ell^2 r  -\| \nabla f(\pw) \|^2/2  \right) \\ 
 & \stackrel{\eqref{eq:sgd_large_gd_condition}}{\leq} - \frac{1}{2} r \| \nabla f(\pw) \|^2 \\ 
 \leq & - \ft 
\end{aligned}
\end{equation}
where the last step is the result of the choice of parameters in Table~\ref{tab:sgd_parameters_with_constraints}.

\paragraph{Sharp curvature regime:} When the minimum eigenvalue is significantly less than zero, SGD steps with a large step-size $r$ provides enough variance for following SGD steps --with a smaller step size $\eta$-- to exploit the negative curvature direction. This statement is formally proved in the next lemma. 
\begin{framed}
\begin{lemma} [Negative curvature exploration by CNC-SGD]\label{lemma:sgd_large_nc}
Suppose Assumptions \ref{ass:CNC} and \ref{ass:smoothness} hold.
If the Hessian matrix at $\pw_t$ has a small negative eigenvalue, i.e.
\begin{align} 
\lambda_{\min} (\hf(\pw_t)) \leq -  \sqrt{\rho}\epsilon.
\end{align}
Then $\tr$ iterations with the small step-size $\eta$ yields the following decrease in the function value in expectation  
\begin{align} 
\E \left[ f(\w_{t+\tr}) \right] - f(\pw_t) \leq - \ft,
\end{align} 
where the expectation is taken over the sequence $\{\w_k\}_{t+1}^{t+k}$. 
\end{lemma}

\end{framed}

\paragraph{Moderate curvature and gradient regime:} Suppose that the minimum eigenvalue of the Hessian is quite small and visited gradients has also a small norm. In this case, we need to bound increase in the function caused by the variance of SGD. The established upperbound in Eq.~\ref{eq:decrease_f_large_gd_sgd} bounds the total increase after $\tr$ iterations as
\begin{equation} \label{eq:increase_bound_small_ss_sgd}
\begin{aligned}
\E \left[ f(\w_{t+1}) \right] - f(\w_t) & \leq L \ell^2 r^2 \\ 
& \leq \delta \ft/2 \quad \quad \quad \quad \text{[The choice of $r$ in Table~\ref{tab:sgd_parameters_with_constraints}]}
\end{aligned} 
\end{equation}
\paragraph{Probabilistic bound} The probabilistic lower bound on returning the desired second order stationary point can be derived from Eq ~\eqref{eq:decrease_f_large_gd_sgd} and \eqref{eq:increase_bound_small_ss_sgd} as well as  Lemma~\ref{lemma:sgd_large_nc} using exactly the same argument as the probabilistic argument on perturbed gradient descent. We define the event $\A_k$ as 
\begin{align*} 
\A_k := \{ \|\gf(\w_{k\tr})\| \geq \epsilon \text{ or } \lambda_{\min}(\hf(\w_{k\tr})) \leq - \sqrt{\rho} \epsilon \}.
\end{align*}
According to the result for the large gradient regime (in Eq.~\eqref{eq:decrease_f_large_gd_sgd}) and the large curvature result (in Lemma~\ref{lemma:sgd_large_nc}), the function value decreases as 
\begin{align*} 
\E \left[ f(\w_{(k+1)\tr}) - f(\w_{k\tr}) | \A_k \right] \leq -\ft. 
\end{align*}
in $\tr$ itetations of SGD. 
In all other iterates, the increase of the function value due to the stochastic gradient steps is controlled by using our choice of steps sizes, according to the result of Eq.~\eqref{eq:increase_bound_small_ss_sgd}: 
\begin{align*} 
\E \left[ f(\w_{(k+1)\tr}) - f(\w_{k\tr}) | \A_k^c \right] \leq \delta \ft/2.
\end{align*}
Let $\P_t$ is the probability associated with $\A_t$, hence $1-\P_t$ is the probability associated with its complement event $\A_t^c$. Note that computing the probabilities $\P_t$ is very hard due to the dependency of $\w_t$ to all stochastic gradient steps before iteration $t$. Plugging these probabilities into the above conditional expectation results yields
\begin{align*} 
\E \left[ f(\w_{(k+1)\tr}) - f(\w_{k\tr})  \right] \leq (1-\P_k) \delta \ft/2 - \P_k \ft.
\end{align*}
Summing the above inequalities over the $K = \lfloor T/\tr\rfloor$ steps obtains the following upper-bound on the average of $\P_t$s
\begin{align*} 
\frac{1}{K}\sum_{k=1}^{K} \P_k & \leq \frac{f(\w_0) - f^*}{K\ft} + \frac{\delta}{2} \\ 
& \leq \delta \quad \quad \quad \quad \text{[The choice of the number of steps $T = K \tr$ in Table~\ref{tab:sgd_parameters_with_constraints}.]}
\end{align*}
The above bound allows us to lower-bound the probability of retrieving an $(\epsilon,\sqrt{\rho}\epsilon)$-second order stationary point (which is equivalent to the occurrence of the complement event $\A_k^c$) uniformly over $T$ steps: 
\begin{align*} 
\sum_{k=1}^K (1-\P_k)/K \geq 1-\delta.
\end{align*}
This concludes the proof of the convergence guarantee of CNC-SGD under Assumption~\ref{ass:CNC}. 
\end{proof} 
\subsection{Proof of the main Lemma~\ref{lemma:sgd_large_nc}}
\textbf{Preliminary} Through our analysis, we invoke the result of the following lemma repeatedly. 
\begin{framed}
\begin{lemma} \label{lemma:zero_covariance_law}
 Consider stochastic processes $X_t$ adapted to filtration $\F_t$ such that $\E \left[ X_t | \F_{t-1} \right]  = 0 $. If the random variable $Y$ is $\F_{t-1}$-measurable, then 
 \begin{align}
     \E \left[ X_t Y \right] = 0  
 \end{align}
\end{lemma}
\end{framed}
\begin{proof}
   The proof is straightforward: 
   \begin{align}
       \E \left[ X_t Y \right] 
       & = \E_{\F_{t-1}} \left[  \E \left[ X_t Y | \F_{t-1}\right]  \right]  \\ 
       & =  \E_{\F_{t-1}} \left[  \E \left[ X_t  | \F_{t-1}\right] Y \right]  = 0 
   \end{align}
 where the last step is due to the fact that $Y$ is $\F_{t-1}$-measurable.  
\end{proof}

\begin{framed}
\begin{lemma} [Restated Lemma~\ref{lemma:sgd_large_nc}]
Suppose Assumptions \ref{ass:CNC} and \ref{ass:smoothness} hold. If the Hessian matrix at $\pw_t$ has a small negative eigenvalue, i.e.
\begin{align} 
\lambda_{\min} (\hf(\pw_t)) \leq -  \sqrt{\rho}\epsilon.
\end{align}
Then $\tr$ iterations with the small step-size $\eta$ yields the following decrease in the function value in expectation  
\begin{align} 
\E \left[ f(\w_{t+\tr}) \right] - f(\pw_t) \leq - \ft,
\end{align} 
where the expectation is taken over the sequence $\{\w_k\}_{t+1}^{t+\tr}$. 
\end{lemma}
\end{framed}

\begin{proof}

Our analysis for the large curvature case in CNC-PGD (lemma~\ref{lemma:small_grad_large_nc_restated}) can be extended to SGD. Here, we borrow the compact notations from Lemma~\ref{lemma:small_grad_large_nc_restated}. Similar to the proof scheme of lemma~\ref{lemma:small_grad_large_nc_restated}, our proof is based on contradiction. We assume that the desired decrease in the  function value is not obtained in $\tr$ iterations, namely 
\begin{align} \label{eq:ass_contradiction_sgd}
\E \left[ f(\w_{t}) \right] - \tilde{f} \geq - \ft, \quad \forall t \leq \tr.
\end{align}
A direct result of the above assumption is that we can establish a bound on the expectation of the distance to $\pw$ for all iterates $\w_t$ such that $t<\tr$.

 \begin{framed}
 \begin{lemma} [Distance Bound for SGD]\label{lemma:expected_distance_bound_sgd}
Suppose that expectation of the decrease in function value is lower-bounded as stated in \eqref{eq:ass_contradiction_sgd}.
Then, the expectation of the distance from the current iterate to $\pw$ is bounded as 
\begin{align*} 
\E \left[ \| \w_t - \pw \|^2 \right] \leq  \left( 4\ft\eta + 2L\eta (\ell r)^2 + 4(\ell \eta)^2  + 2 L \eta^3 \ell^2 \tr \right) t+  2 (\ell r)^2,
\end{align*}  
for all $t< \tr$ as long as Assumption~\ref{ass:smoothness} holds.
\end{lemma}
\end{framed}

We postpone the proof of the last lemma to Section C.4. The proposed bound in the last lemma is larger than the established distance bound for PGD steps (lemma~\ref{lemma:expected_distance_bound}). This is due to the variance of the stochastic gradients. In the rest of the proof, we will construct a lower-bound that contradicts to the above upper-bound using the CNC assumption~\ref{ass:CNC}.  
To this end, we expand the SGD steps (in Eq.~\eqref{eq:sgd_step_expansion}) using the gradient of the stale Taylor expansion $g(\pw)$:
\begin{equation} \label{eq:sgd_iterates_decomposition}
\begin{aligned} 
\w_{t+1} - \pw & = \w_{t} - \pw - \eta \gf_t + \eta \zeta_t \\ 
 &  = \w_t - \pw - \eta \gg_t
 + \eta \left( \gf_t - \gg_t -\gf(\pw) + \zeta_t \right) \\ 
 & = \left( \Im - \eta \H  \right) (\w_t - \pw) + \eta  \left( \gf_t - \gg_t - \gf(\pw) + \zeta_t \right) \\ & = \u_t + \eta \left( \bdelta_t + \d_t + \bzeta_t \right) ,
 \end{aligned} 
 \end{equation}
 where the vectors $\u_t$, $\bdelta_t$, $\d_t$, and $\bzeta_t$ are defined in Table~\ref{tab:sgd_expansion_vectors}. The only new term in the expansion is the noise of the stochastic gradient steps $\bzeta_t$s. 
\begin{table}[h!]
    \centering
    \begin{tabular}{c c c c}
        Vector  & Formula & Indication & Included in PGD analysis?\\
        \hline
         $\u_t$ & $\left(\Im - \eta \H \right)^t(\w_1-\pw)$ & Power Iteration & Yes\\ 
         $\bdelta_t$ & $\sum_{i=1}^t \left( \Im - \eta \H \right)^{t-i} \left( \gf_t - \gg_t \right)$ & Stale Taylor Approximation Error & Yes \\ 
         $\d_t$ & $ -\sum_{i=1}^t \left( \Im - \eta \H \right)^{t-i} \gf(\pw)$ & Initial Gradient Dependency & Yes \\ 
         $\bzeta_t$ & $\sum_{i=1}^t \left( \Im - \eta \H  \right)^{t-i} \zeta_i$ & Noise of Stochastic Gradients & No
    \end{tabular}
    \caption{Components of CNC-SGD expanded steps.}
    \label{tab:sgd_expansion_vectors}
\end{table}
Similarly to PGD, the power iterations $\u_t$ plays an essential rule in the negative curvature exploration. For this term, we can reuse our analysis from lemma~\ref{lemma:upper_bound_u} and \ref{lemma:lower_bound_u_2}. The term $\bdelta_t$ is caused by using a stale Taylor approximation for all iterates $t\leq \tr$. We need to bound the perturbation effect of this term to guarantee that power iterates $\u_t$ exploit the negative curvature. To this end, we required a bound on $\E\| \bdelta_t\|$. This bound is established in the next lemma using the distance bound of Lemma~\ref{lemma:expected_distance_bound_sgd}. 
\begin{framed}
\begin{lemma} \label{lemma:delta_bound_sgd} 
Under the condition of Lemma~\ref{lemma:expected_distance_bound_sgd}, the bound 
\begin{align} 
\E \| \bdelta_t \| \leq \rho \left( \frac{2 (\ell r)^2}{\eta \lambda} 
+ \frac{4 \eta\ft + 2L \eta (\ell r)^2 + 4(\ell \eta)^2 + 6 L\eta^3 \ell^2 \tr }{(\lambda\eta)^2}    \right) \kappa^t
\end{align} 
holds true.
\end{lemma}
\end{framed}

\begin{proof}
\begin{equation}
\begin{aligned} 
\E \| \bdelta_t \| & = \E \| \sum_{k=1}^t \left( \Im - \eta \H  \right)^{t-k} \left( \gf_k - \gg_k \right) \|  \\ 
& \leq \sum_{k=1}^t (1+\eta  \lambda )^{t-k} \E \| \gf_k - \gg_k \| \\ 
& \leq (\rho/2) \sum_{k=1}^t \kappa^{t-k}  \E \| \w_k - \pw \|^2 \\ 
& \leq (\rho/2) \sum_{k=1}^t \kappa^{t-k}  \left(\left( 4\ft\eta + 2L\eta (\ell r)^2 + 4(\ell \eta)^2  + 2 L \eta^3 \ell^2 \tr \right)t + 2 (\ell r)^2  \right) \quad \text{\tiny{[Lemma~\ref{lemma:expected_distance_bound_sgd}]}} \\ 
& \leq \rho \left( \frac{2 (\ell r)^2}{\eta \lambda} 
+ \frac{4 \eta\ft + 2L \eta (\ell r)^2 + 4(\ell \eta)^2 + 6 L\eta^3 \ell^2 \tr }{(\lambda\eta)^2}    \right) \kappa^t  \quad \text{\tiny{[Lemma~\ref{lemma:auxiliary}]}}
\end{aligned}
\end{equation}
\end{proof}
 \paragraph{Lower-bound on the distance}
 Using the step expansion, we lower-bound the distance from the pivot $\pw$ as 
 \begin{equation} 
 \begin{aligned} 
 \E \| \w_{t+1} - \pw \|^2 = \E \| \u_t \|^2 - 2 \eta \| \u_t \| \E \| \bdelta_t \| + 2 \eta \E \left[ \u_t^\top \d_t \right] + 2 \eta \E \left[ \u_t^\top  \bzeta_t \right]
 \end{aligned}
 \end{equation}
where $\Omega_t =  \bdelta_t + \bzeta_t + \d_t$.
The result of lemma~\ref{lemma:zero_covariance_law} implies that that $\E \left[ \u_t^\top \bzeta_t \right] = 0 $. Plugging the result into the above equation yields. 
  \begin{equation} 
 \begin{aligned} 
 \E \| \w_{t+1} - \pw \|^2 &=  \E \| \u_t \|^2 - 2 \eta \| \u_t \| \E \| \bdelta_t \| + 2 \eta \E \left[ \u_t^\top \d_t \right] \\ 
 & \geq \E \| \u_t \|^2 - 2 \eta \| \u_t \| \E \| \bdelta_t \| + 2 \eta \E \left[ \u_t^\top \d_t \right] \\ 
 & \geq \E \| \u_t \|^2 - 2 \eta \| \u_t \| \E \| \bdelta_t \| \quad \text{\tiny{[Lemma~\ref{lemma:removing_initial_gradient_dependency}]}}\\
 & \geq \gamma r^2 \kappa^{2t} - 2 \eta \ell r \kappa^t \E \left[ \| \bdelta_t \| \right]\quad \text{\tiny{[Lemma~\ref{lemma:lower_bound_u_2} \& \ref{lemma:upper_bound_u}]}}\\ 
 & \geq \left( \gamma r -  \frac{4 \rho \ell^3 r^2}{\lambda}  - \frac{ 8 \rho \ell \ft}{\lambda^2 } - \frac{ 4 L \rho \ell^3 r^2 }{\lambda^2} -\frac{8  \rho \ell^3 \eta   }{\lambda^2}- \frac{12 L\rho\ell^3 \eta^2 \tr  }{\lambda^2}\right) r \kappa^{2t} \quad\text{\tiny{[Lemma~\ref{lemma:delta_bound_sgd}]}} \label{eq:lower-bound-distance}
 \end{aligned}
 \end{equation}

 \paragraph{Constraints on the parameter} To derived the desired contradiction, i.e. $\E \left[ f_t \right] - \tilde{f} \leq -\ft$, we need to prove that the distance $\E\left[ \| \w_{t} - \pw \|^2 \right]$ is larger than the upper-bound established in lemma~\ref{lemma:expected_distance_bound_sgd}. To this end, we have to choose parameters such that the established lower-bound on the distance in Eq.~\eqref{eq:lower-bound-distance} be positive, i.e.   
 \begin{align} 
 \gamma r -  \underbrace{\frac{4\rho \ell^3 r^2}{\lambda}}_{\leq \gamma r/6}  
 - \underbrace{\frac{8 \rho \ell \ft}{\lambda^2 }}_{\leq \gamma r/6} 
 - \underbrace{\frac{ 4 \rho L \ell^3 r^2 }{\lambda^2}}_{\leq \gamma r/6} 
 -\underbrace{\frac{8 \rho \ell^3 \eta  }{\lambda^2}}_{\leq \gamma r/6} 
 - \underbrace{\frac{12  \ell^3 \rho L\eta^2 \tr}{\lambda^2}}_{\leq \gamma r/6}  \stackrel{!}{\geq} \gamma r/6
 \end{align}
 Using the lower-bound on the absolute value of minimum eigenvalue, i.e. $\lambda \geq \sqrt{\rho}\epsilon$, we choose parameters such that the above constraints are satisfied (eee Table~\ref{tab:sgd_parameters_with_constraints}:
  \begin{align} 
 r &  \leq (\gamma \lambda^2)/(24 L\rho \ell^3 ) \leq \lambda \gamma/(24 \rho \ell^3) \label{eq:sgd_r_upperbound}\\
 \ft & \leq \gamma \lambda^2r/(48\rho \ell). \label{eq:sgd_ft_bound}
 \end{align}
 Furthermore, $\eta$ has to be bounded as
 \begin{equation} \label{eq:sgd_eta_bound}
 \eta \leq  (\gamma r /(72 L \rho \tr \ell^3))^{1/2} \lambda \leq \gamma r\lambda^2 /( 48 \ell^3\rho L) 
 \end{equation}
 We postpone the choose the step-size after finding an expression for $\tr$.

 Our choice of parameters fulfills the above constraints. Plugging the above result into Eq.~\eqref{eq:lower-bound-distance} obtains the exponential growing lower-bound on the distance 
 \begin{align} \label{eq:postive_distance_lower_bound_sgd}
 \E \left[ \| \w_{t+1} - \tilde{\w} \| \right] \geq  \gamma r^2 \kappa^{2t}/6
 \end{align} 
 \paragraph{Contradiction by choosing the number of iterations $\tr$}
 Using the lower bound of Eq.~\eqref{eq:postive_distance_lower_bound_sgd}, we can establish a contradictory result with the upperbound on the distance proposed in lemma~\ref{lemma:expected_distance_bound_sgd} 
 \begin{align*} 
 \gamma r^2 \kappa^{2\tr}/6 \stackrel{!}{\geq}    \left( 2\ft\eta + L\eta (\ell r)^2 + 2(\ell \eta)^2  \right) t+  L \eta^3 \ell^2 t^2+ 2 (\ell r)^2.  
 \end{align*}
 Since the left-side of the above inequality is exponentially growing, one can choose the number iterations $\tr$ large enough to derive the contradiction: 
 \begin{align} \label{eq:sgd_tr_bound}
 \tr \geq c (\eta \lambda )^{-1} \log(L\ell/(\gamma r\eta\lambda))
 \end{align}
 where $c$ is a constant independent of parameters $\lambda$,$\gamma$,$L$ and $\rho$.
 \paragraph{The choice of step-size} 
 To derive our contradiction, we required the upperbound of Eq.~\eqref{eq:sgd_eta_bound} to be statisfied. Replacing $\tr$, achieved in Eq.~\eqref{eq:sgd_tr_bound}, into this bound yields 
 \begin{align} \label{eq:sgd_eta_bound_2}
     \eta \leq c' \gamma  r \lambda^3 /(72 L \rho \ell^3)
 \end{align}
 where $c' > \max \{ \left( c \log(L\ell/(\gamma r \lambda)\right)^{-1}, 1 \}$. The choice of step-size in Table~\ref{tab:sgd_parameters_with_constraints} ensures the above inequality holds.  
 \end{proof}
 
 \subsection{Bound on the expectation of distance}
 Here, we complete the proof of lemma~\ref{lemma:small_grad_large_nc_restated} by proving the following lemma, which is used in lemma~\ref{lemma:small_grad_small_nc_restrated}. 
 \begin{framed}
 \begin{lemma} [Restated Lemma~\ref{lemma:expected_distance_bound_sgd}]
Suppose that expectation of the decrease in function value is lower-bounded as 
\begin{align} \label{eq:ass_contradiction_sgd_restated}
\E \left[ f(\w_{\tr}) \right] - \tilde{f} \geq - \ft.
\end{align}
Then, the expectation of the distance from the current iterate to $\pw$ is bounded as 
\begin{align} 
\E \left[ \| \w_t - \pw \|^2 \right] \leq  \left( 4\ft\eta + 2L\eta (\ell r)^2 + 4(\ell \eta)^2  + 2 L \eta^3 \ell^2 \tr \right) t+  2 (\ell r)^2,
\end{align} 
for $t\leq \tr$ as long as Assumption~\ref{ass:smoothness} holds.
\end{lemma}
\end{framed}
 
\begin{proof}
We use the result of lemma~\ref{lemma:decrease_f_sgd} 
\begin{align*}
  -\ft \leq \E \left[  f_{\tr} - \tilde{f} \right] & =  \E \left[ \sum_{i=0}^{\tr-1} f_{i+1}  -  f_i\right] \\ 
  & \leq -\eta \sum_{i=0}^{\tr-1} \E \|\gf_i \|^2 +  L (\ell \eta)^2 \tr/2+ L (\ell r)^2/2  \quad \text{\tiny{[Lemma~\ref{lemma:decrease_f_sgd}]}}.
\end{align*}
Rearranging terms obtains a bound on the sum of the squared norm of visited gradients:
\begin{align}  \label{eq:gradient_sum_bound_sgd}
\sum_{i=1}^{\tr} \E \| \gf_i \|^2 \leq \ft/\eta +  L \ell^2 \eta \tr/2+ L(\ell r)^2/(2\eta)
\end{align} 
Using the Telescopic expansion of the difference $\w_{t+1} - \w_1$, we relate the distance to the visited stochastic gradients:
\begin{equation}
\begin{aligned} 
\E \left[\| \w_{t+1} - \w_1 \|^2  \right] & = \E \left[ \| \sum_{i=1}^{t} \w_{i+1} - \w_{i} \|^2  \right] \\
& \leq  \eta^2 \E \|\sum_{i=1}^{t}  \left( \zeta_i- \gf_i \right)  \|^2. \quad \text{\tiny{[SGD-step decomposition, Eq.~\eqref{eq:sgd_step_expansion}]}} \label{eq:sgd_telescopic_distance_expansion}
\end{aligned}
\end{equation}

To upper bound the right-side of the above inequality, we rely zero-mean assumption of $\zeta_t$s: 
\begin{equation}  \label{eq:upperbound_zeta_gf}
\begin{aligned} 
\E \|\sum_{i=1}^{t}  (\zeta_i -\gf_i)  \|^2 & \leq 2 \E \| \sum_{i=1}^{t} \gf_i \|^2 + 2 \E \|\sum_{i=1}^{t} \zeta_i \|^2  \quad \text{\tiny{[Parallelogram law]}}\\
&= 2\E \| \sum_{i=1}^t \gf_i \|^2 + 2\sum_{i\neq j} \E  \left[\zeta_i^\top \zeta_j\right] + 2\sum_{i=1}^t \E \left[ \zeta_i^\top \zeta_i\right]
\end{aligned}
\end{equation}
To further simplify the above bound, we invoke the result of lemma~\ref{lemma:zero_covariance_law} which proves that  $\E\left[ \zeta_i^\top \zeta_j \right] = 0$.
Replacing this result into Eq.~\eqref{eq:upperbound_zeta_gf} yields
\begin{equation}
\begin{aligned} 
\E \|\sum_{i=1}^{t}  (\zeta_i -\gf_i)  \|^2 & = 2\E \| \sum_{i=1}^t \gf_i \|^2 + 2\sum_{i=1}^t \E \| \zeta_i \|^2 \\ 
& \leq 2\E \| \sum_{i=1}^t \gf_i \|^2 + 2 t \ell^2  \\ 
& \leq 2 \E \left( \sum_{i=1}^t \| \gf_i\| \right)^2 + 2 t \ell^2 \quad \text{\tiny{[Triangle inequality]}} \\ 
& \leq 2 t \sum_{i=1}^t \E \| \gf_i \|^2 + 2 t \ell^2 \quad \text{\tiny{[Cauchy–Schwarz inequality]}} \\ 
& \stackrel{\eqref{eq:gradient_sum_bound_sgd}}{\leq} 2 t \left( \ft/\eta +  L \eta \ell^2 \tr/2+ L(\ell r)^2/(2\eta) + \ell^2  \right) .
\end{aligned}
\end{equation}
Replacing the above bound  into Eq.~\eqref{eq:sgd_telescopic_distance_expansion} yields: 
\begin{align*} 
\E \| \w_{t+1} - \w_1 \|^2 \leq  t \left(2\ft\eta + L\eta (\ell r)^2 + 2(\ell \eta)^2  \right) +  L \eta^3 \ell^2 t \tr.
\end{align*}

Using the above result, we bound the distance as:
\begin{equation}
\begin{aligned} 
\E \| \w_{t+1} - \pw \|^2 & \leq  2 \E \| \w_{t+1} - \w_1 \|^2 + 2 \E \left[ \| \w_{1} - \pw\|^2 \right]  \quad \text{\tiny{[Parallelogram law]}} \\
& \leq 2 \E \| \w_{t+1} - \w_1 \|^2 + 2(\ell r)^2 \\ 
& \leq  \left( 4\ft\eta + 2L\eta (\ell r)^2 + 4(\ell \eta)^2  \right) t+  2 L \eta^3 \ell^2 t \tr+ 2 (\ell r)^2
\end{aligned} 
\end{equation}

Finally, replacing $t+1$ by $t$ concludes the proof.
\end{proof}

 \section{Analysis of Learning Half-spaces}
 
\begin{framed}
\begin{lemma}[Restated~\ref{lemma:CNC_lowerbound}]
Consider the problem of learning half-spaces as stated in Eq. (\ref{eq:half_space}), where $\varphi$ satisfies Assumption~\ref{eass:seim-self-concordant}. Furthermore, assume that the support of $\P$ is a subset of the unit sphere. Let $\v$ be a unit length eigenvector of $\hf(\w)$ with corresponding eigenvalue $\lambda<0$. Then 
\begin{align} 
\E_\z \left[ (\gf_{\z}(\w)^\top\v)^2 \right] \geq  (\lambda/c)^2. 
\end{align} 
\end{lemma} 
\end{framed}

\begin{proof} 
Using the definition of an eigenvector,  $\nabla^2f(\w)\v=\lambda \v$ and since $\nabla^2f(\w)=\varphi''(\w^\top \z)\z\z^\top$ we have:
\begin{equation}
\begin{aligned}
\lambda & = \v^\top \hf(\w) \v  \\
& = \E \left[ \varphi''(\w^\top \z) (\z^\top \v)^2 \right]   \\ 
& \geq - \E \left[ | \varphi''(\w^\top \z)|  (\z^\top \v)^2 \right]   \\ 
& \geq- c \E \left[ |\varphi'(\w^\top \z) |(\z^\top \v)^2 \right] \qquad \tiny{\text{[Eq.~(\ref{eq:seim-self-concordant}])}}  \\ 
& \geq - c \E \left[ |\varphi'(\w^\top \z)| |\z^\top \v|\right] \qquad \tiny{\text{[$\|\z\|\leq 1$]}} \\ 
& \geq - c \E \left[ |\varphi'(\w^\top \z) \z^\top \v|\right].
\label{eq:result_of_property}
\end{aligned}
\end{equation}
Using the above result and as well as Jensen's inequality, we derive the desired result:
\begin{equation}
\begin{aligned} 
\E \left[ (\gf_\z(\w)^\top \v)^2 \right] & = \E \left[ (\varphi'(\w^\top \z)\z^\top\v)^2  \right]  \\ 
& = \left( \sqrt{\E \left[ (\varphi'(\w^\top \z)\z^\top\v)^2 \right]} \right)^2  \\ 
&  \geq \left( \E \left[ \sqrt{(\varphi'(\w^\top \z) \z^\top\v)^2}  \right] \right)^2  \\ 
& \geq \left( \E | \varphi'(\w^\top \z) \z^\top\v)|\right)^2  \\
& \geq (\lambda/c)^2,
\end{aligned}
\end{equation}
where the last inequality follows from Eq. (\ref{eq:result_of_property}) and the fact that $\lambda <0$.
\end{proof}

 \section{Additional experimental results}
\paragraph{Learning halfspaces}
From each of two multivariat gaussian distributions we draw $n/2=20$ samples $x_i\in \mathbb{R}^4$ and assign them the labels $y_i\in\{0,1\}$ respectively. We then optimize the loss function

\begin{equation*}
    f(\w)= \text{sigmoid} \left(-y_i\x_i^\top \w\right)+\frac{1}{2} \|\w\|^2
\end{equation*}
with the following methods and hyperparameters:

Gradient Descent, Stochastic Gradient Descent, PGD as in \cite{jin2017escape} with perturbation radius $r=0.1$ and PGD-CNC with a stochastic gradient step as perturbation. All methods use the step size $\alpha=1/4$, the stochastic gradient steps are performed with batch size $1$ and the perturbed gradient descent methods perturb as soon as $\gf(\w)<\gt:=0.01$.

To complete the picture of Figure \ref{fig:escaping_saddles} we here also present the gradient norms and minimum/maximum eigenvalues along the trajectories of the different methods. It becomes apparent that all of them indeed started at a saddle and eventually move towards (and along) the flat end of the sigmoid. However, Gradient Descent is much slower in finding regions of significant negative curvature than the stochastic methods.

\begin{figure*}[h!]
	\begin{center}
          \begin{tabular}{@{}c@{\hspace{2mm}}c@{\hspace{2mm}}c@{\hspace{2mm}}c@{}}
          \vspace{-5.5pt}
            \includegraphics[width=0.33\linewidth]{experiments/loss.png} &
            \includegraphics[width=0.33\linewidth]{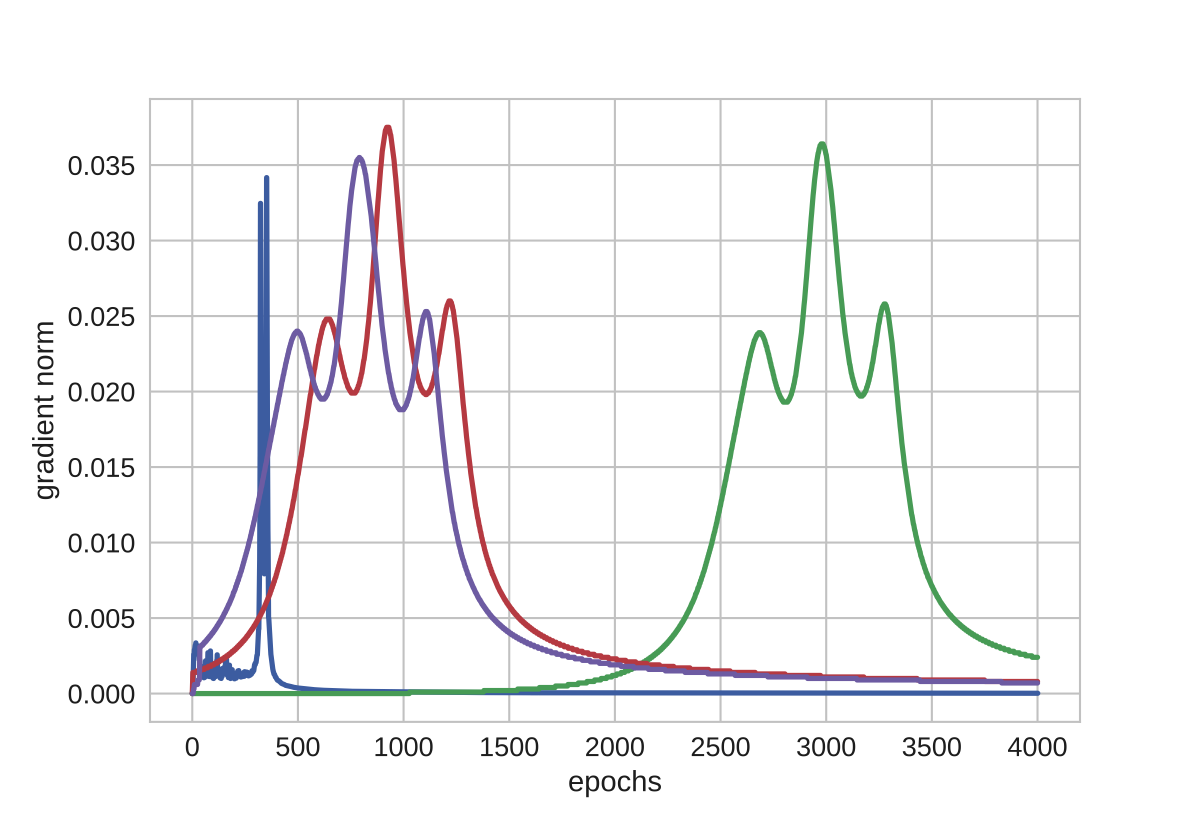} & \includegraphics[width=0.33\linewidth]{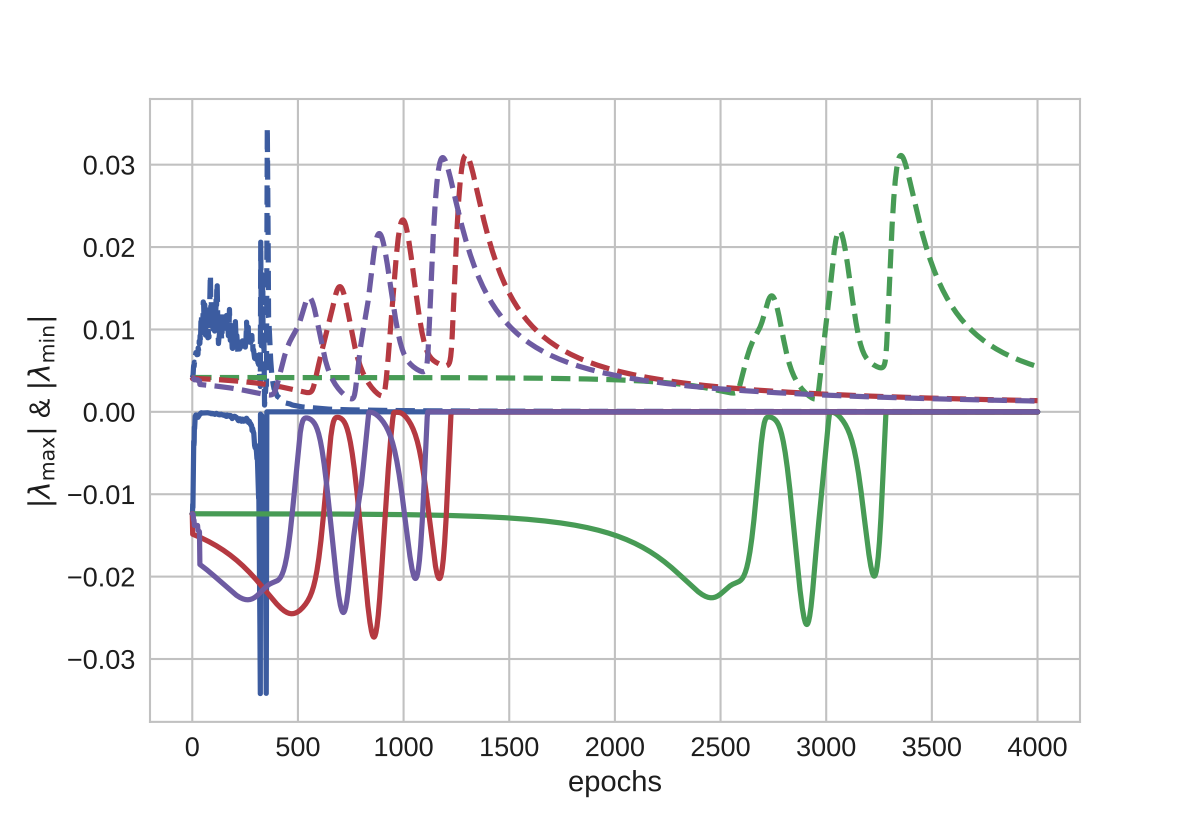} \\  \vspace{-1pt}
            
         	1. \footnotesize{{ Suboptimality}} &
            2. \footnotesize{{ Gradient norm}} &
            3. \footnotesize{{ min and max eigenvalues}}
            
	  \end{tabular}
          \caption{Learning halfspaces: more details.}
\label{fig:learning_halfspaces_details}

	\end{center}
\end{figure*}
 
 \paragraph{Neural Networks}
 The neural network experiments were implemented using the Pytorch library and conducted on a GPU server. Note that we downsized the mnist dataset to an image size of $10\times 10$ and applied sigmoid acivations in the hidden layers as well as a cross-entropy loss over the 10 classes.
 
 While we present covariances between the stochastic gradients/isotropic noise vectors with the \textit{leftmost} Eigenvectors in the main paper, Figure \ref{fig:cov_all} plots the covariances with the entire negative eigenspectrum. 
\begin{figure*}[h!]
	\begin{center}
          \begin{tabular}{@{}c@{\hspace{2mm}}c@{\hspace{2mm}}c@{\hspace{2mm}}c@{}}
          \vspace{-5.5pt}
            \includegraphics[width=0.33\linewidth]{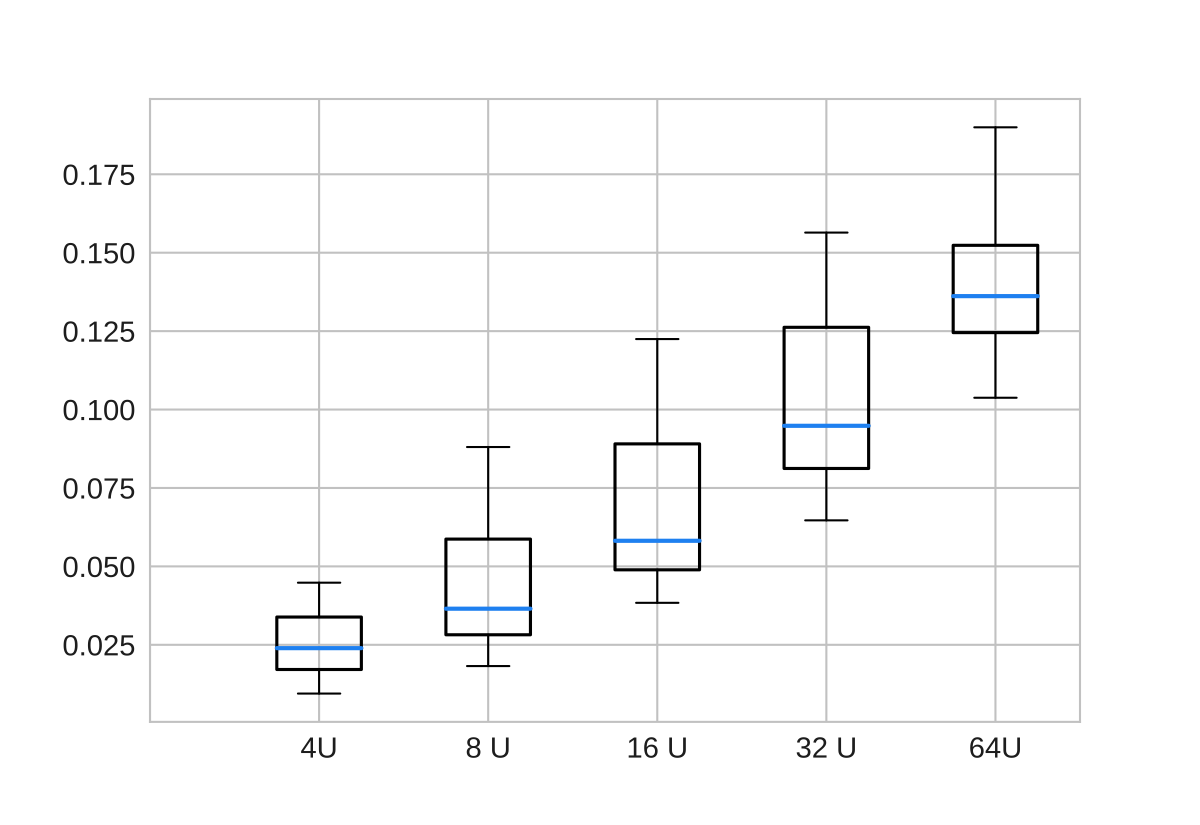} &
            \includegraphics[width=0.33\linewidth]{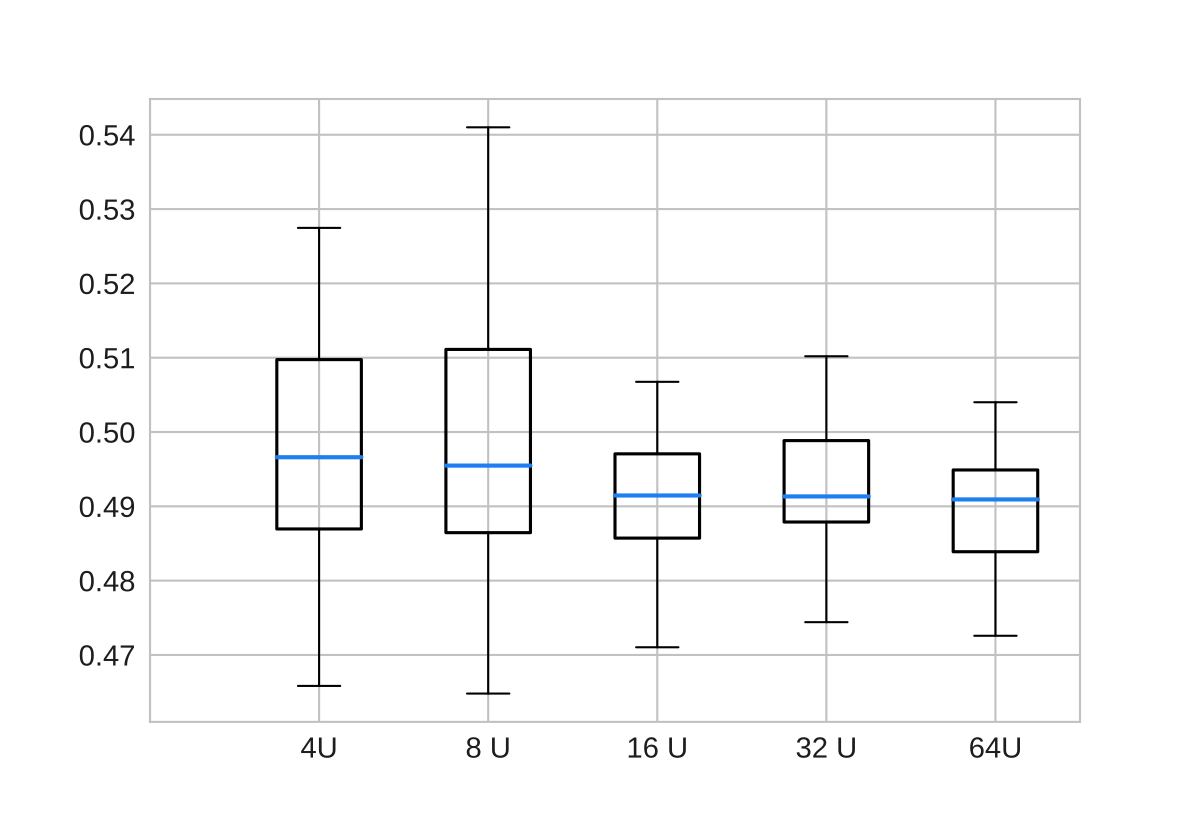} & \includegraphics[width=0.33\linewidth]{experiments/EVs_1HL.png} \\  \vspace{-1pt}
            
            \includegraphics[width=0.33\linewidth]{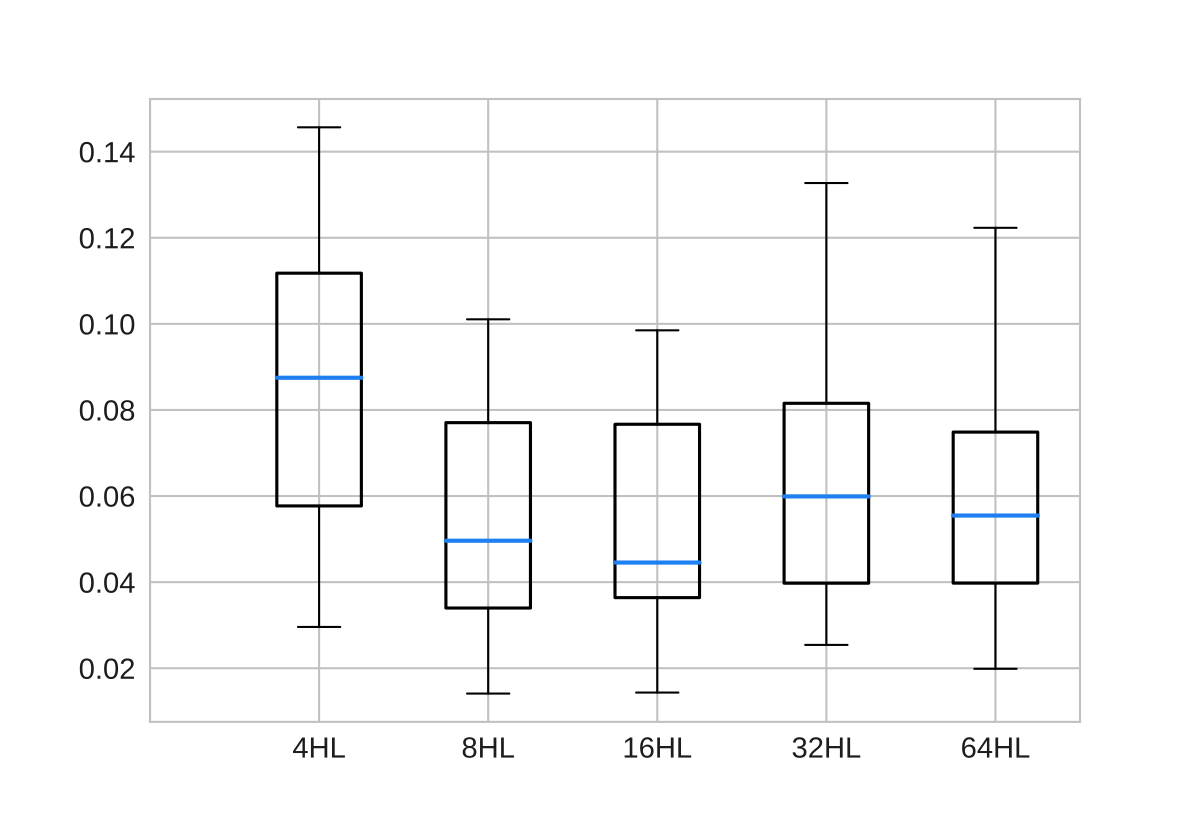} &
            \includegraphics[width=0.33\linewidth]{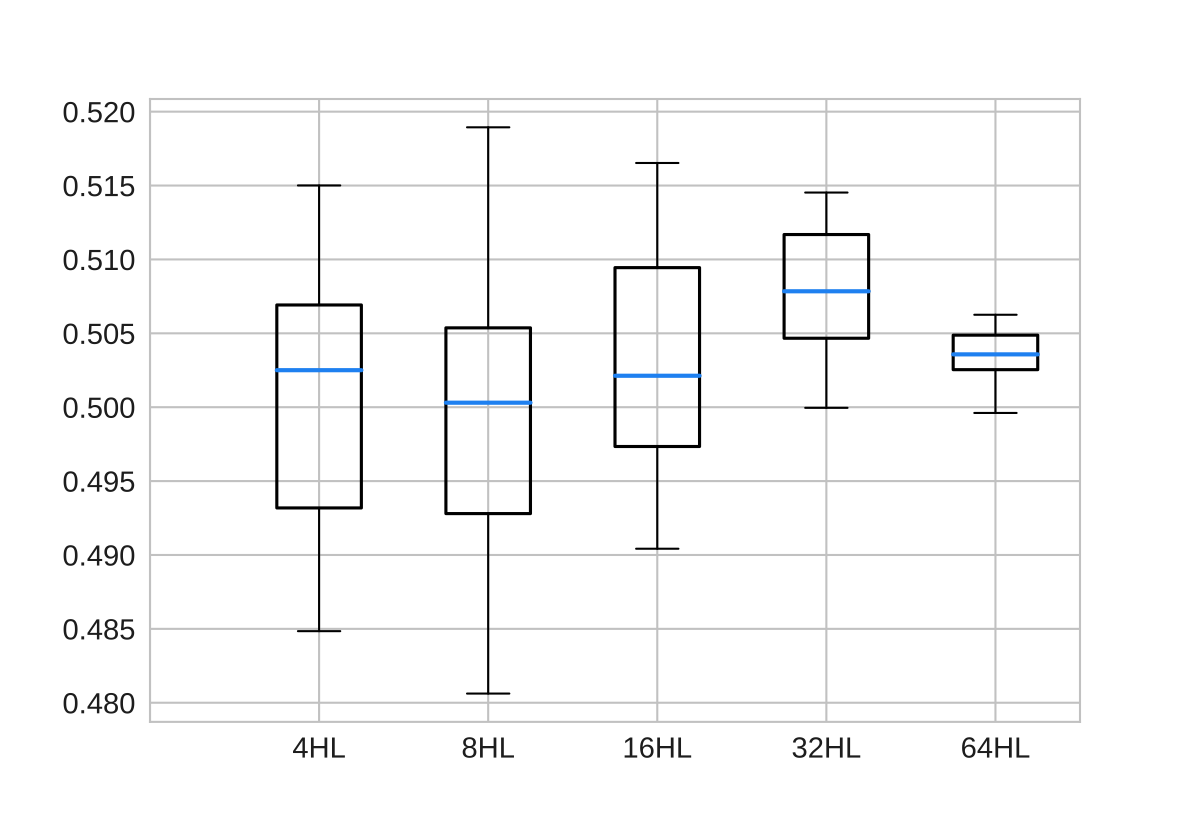} & \includegraphics[width=0.33\linewidth]{experiments/EVs_xHL.png} \\  \vspace{-1pt}
         	1. \footnotesize{{ SGD covariance w/ left eigenvectors}} &
            2. \footnotesize{{ isotropic noise covariance w/ left eigenvectors}} &
            3. \footnotesize{{min and max eigenvalues}}
            
	  \end{tabular}
          \caption{Average covariances and eigenvalues of 30 random parameters in Neural Networks with increasing width (top) and depth (bottom).}
	  \label{fig:cov_all}

	\end{center}
\end{figure*}

 In Figure \ref{fig:cov_over_ev} we show that the correlation of eigenvectors and stochastic gradients increases with the magnitude of the associated eigenvalues. As expected, this is not the case for noise vectors that are drawn randomly from the unit sphere. Furthermore, these correlations show a decrease with an increasing dimension as can be seen in Figure \ref{fig:cov_over_ev_uni}. 

\begin{figure}
\centering
\includegraphics[width=200pt]{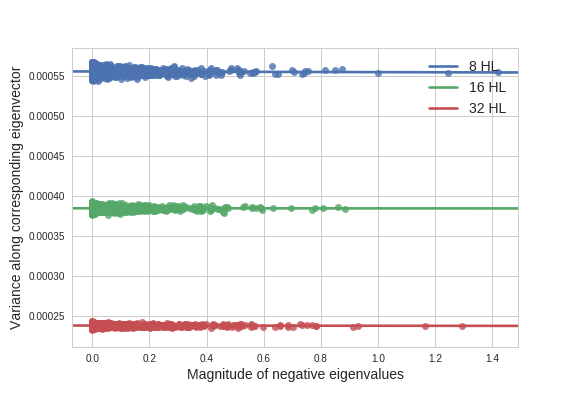}
\caption{Correlation of stochastic gradients with eigenvectors corresponding to eigenvalues of different magnitudes on Neural Nets with 8, 16 and 32 hidden layers. Scatterplot and fitted linear model with 95\% confidence interval.}
\label{fig:cov_over_ev_uni}
\end{figure}

%% file: noise_sgd.bbl
\begin{thebibliography}{30}
\providecommand{\natexlab}[1]{#1}
\providecommand{\url}[1]{\texttt{#1}}
\expandafter\ifx\csname urlstyle\endcsname\relax
  \providecommand{\doi}[1]{doi: #1}\else
  \providecommand{\doi}{doi: \begingroup \urlstyle{rm}\Url}\fi

\bibitem[Allen-Zhu(2017)]{allen2017natasha}
Allen-Zhu, Z.
\newblock Natasha 2: Faster non-convex optimization than sgd.
\newblock \emph{arXiv preprint arXiv:1708.08694}, 2017.

\bibitem[Allen-Zhu \& Li(2017)Allen-Zhu and Li]{allen2017neon2}
Allen-Zhu, Z. and Li, Y.
\newblock Neon2: Finding local minima via first-order oracles.
\newblock \emph{arXiv preprint arXiv:1711.06673}, 2017.

\bibitem[Bottou(2010)]{bottou2010large}
Bottou, L.
\newblock Large-scale machine learning with stochastic gradient descent.
\newblock In \emph{Proceedings of COMPSTAT'2010}, pp.\  177--186. Springer,
  2010.

\bibitem[Carmon et~al.(2017)Carmon, Hinder, Duchi, and
  Sidford]{carmon2017convex}
Carmon, Y., Hinder, O., Duchi, J.~C., and Sidford, A.
\newblock " convex until proven guilty": Dimension-free acceleration of
  gradient descent on non-convex functions.
\newblock \emph{arXiv preprint arXiv:1705.02766}, 2017.

\bibitem[Cartis et~al.(2012)Cartis, Gould, and Toint]{cartis2012much}
Cartis, C., Gould, N.~I., and Toint, P.~L.
\newblock \emph{How Much Patience to You Have?: A Worst-case Perspective on
  Smooth Noncovex Optimization}.
\newblock Science and Technology Facilities Council Swindon, 2012.

\bibitem[Chaudhari \& Soatto(2017)Chaudhari and
  Soatto]{chaudhari2017stochastic}
Chaudhari, P. and Soatto, S.
\newblock Stochastic gradient descent performs variational inference, converges
  to limit cycles for deep networks.
\newblock \emph{arXiv preprint arXiv:1710.11029}, 2017.

\bibitem[Choromanska et~al.(2015)Choromanska, Henaff, Mathieu, Arous, and
  LeCun]{choromanska2015loss}
Choromanska, A., Henaff, M., Mathieu, M., Arous, G.~B., and LeCun, Y.
\newblock The loss surfaces of multilayer networks.
\newblock In \emph{AISTATS}, 2015.

\bibitem[Conn et~al.(2000)Conn, Gould, and Toint]{conn2000trust}
Conn, A.~R., Gould, N.~I., and Toint, P.~L.
\newblock \emph{Trust region methods}.
\newblock SIAM, 2000.

\bibitem[Curtis \& Robinson(2017)Curtis and Robinson]{curtis2017exploiting}
Curtis, F.~E. and Robinson, D.~P.
\newblock Exploiting negative curvature in deterministic and stochastic
  optimization.
\newblock \emph{arXiv preprint arXiv:1703.00412}, 2017.

\bibitem[Dauphin et~al.(2014)Dauphin, Pascanu, Gulcehre, Cho, Ganguli, and
  Bengio]{dauphin2014identifying}
Dauphin, Y.~N., Pascanu, R., Gulcehre, C., Cho, K., Ganguli, S., and Bengio, Y.
\newblock Identifying and attacking the saddle point problem in
  high-dimensional non-convex optimization.
\newblock In \emph{Advances in neural information processing systems}, pp.\
  2933--2941, 2014.

\bibitem[Ge et~al.(2015)Ge, Huang, Jin, and Yuan]{ge2015escaping}
Ge, R., Huang, F., Jin, C., and Yuan, Y.
\newblock Escaping from saddle points-online stochastic gradient for tensor
  decomposition.
\newblock In \emph{COLT}, pp.\  797--842, 2015.

\bibitem[Ghadimi \& Lan(2013)Ghadimi and Lan]{ghadimi2013stochastic}
Ghadimi, S. and Lan, G.
\newblock Stochastic first-and zeroth-order methods for nonconvex stochastic
  programming.
\newblock \emph{SIAM Journal on Optimization}, 23\penalty0 (4):\penalty0
  2341--2368, 2013.

\bibitem[Goyal et~al.(2017)Goyal, Doll{\'a}r, Girshick, Noordhuis, Wesolowski,
  Kyrola, Tulloch, Jia, and He]{goyal2017accurate}
Goyal, P., Doll{\'a}r, P., Girshick, R., Noordhuis, P., Wesolowski, L., Kyrola,
  A., Tulloch, A., Jia, Y., and He, K.
\newblock Accurate, large minibatch sgd: training imagenet in 1 hour.
\newblock \emph{arXiv preprint arXiv:1706.02677}, 2017.

\bibitem[Hillar \& Lim(2013)Hillar and Lim]{hillar2013most}
Hillar, C.~J. and Lim, L.-H.
\newblock Most tensor problems are np-hard.
\newblock \emph{Journal of the ACM (JACM)}, 60\penalty0 (6):\penalty0 45, 2013.

\bibitem[Jin et~al.(2017{\natexlab{a}})Jin, Ge, Netrapalli, Kakade, and
  Jordan]{jin2017escape}
Jin, C., Ge, R., Netrapalli, P., Kakade, S.~M., and Jordan, M.~I.
\newblock How to escape saddle points efficiently.
\newblock \emph{arXiv preprint arXiv:1703.00887}, 2017{\natexlab{a}}.

\bibitem[Jin et~al.(2017{\natexlab{b}})Jin, Netrapalli, and
  Jordan]{jin2017accelerated}
Jin, C., Netrapalli, P., and Jordan, M.~I.
\newblock Accelerated gradient descent escapes saddle points faster than
  gradient descent.
\newblock \emph{arXiv preprint arXiv:1711.10456}, 2017{\natexlab{b}}.

\bibitem[Johnson \& Zhang(2013)Johnson and Zhang]{johnson2013accelerating}
Johnson, R. and Zhang, T.
\newblock Accelerating stochastic gradient descent using predictive variance
  reduction.
\newblock In \emph{Advances in Neural Information Processing Systems}, pp.\
  315--323, 2013.

\bibitem[Kohler \& Lucchi(2017)Kohler and Lucchi]{kohler2017sub}
Kohler, J.~M. and Lucchi, A.
\newblock Sub-sampled cubic regularization for non-convex optimization.
\newblock In \emph{International Conference on Machine Learning}, 2017.

\bibitem[Lee et~al.(2016)Lee, Simchowitz, Jordan, and Recht]{lee2016gradient}
Lee, J.~D., Simchowitz, M., Jordan, M.~I., and Recht, B.
\newblock Gradient descent converges to minimizers.
\newblock \emph{arXiv preprint arXiv:1602.04915}, 2016.

\bibitem[Levy(2016)]{levy2016power}
Levy, K.~Y.
\newblock The power of normalization: Faster evasion of saddle points.
\newblock \emph{arXiv preprint arXiv:1611.04831}, 2016.

\bibitem[Moulines \& Bach(2011)Moulines and Bach]{moulines2011non}
Moulines, E. and Bach, F.~R.
\newblock Non-asymptotic analysis of stochastic approximation algorithms for
  machine learning.
\newblock In \emph{Advances in Neural Information Processing Systems}, pp.\
  451--459, 2011.

\bibitem[Nesterov(2013)]{nesterov2013introductory}
Nesterov, Y.
\newblock \emph{Introductory lectures on convex optimization: A basic course},
  volume~87.
\newblock Springer Science \& Business Media, 2013.

\bibitem[Nesterov \& Polyak(2006)Nesterov and Polyak]{nesterov2006cubic}
Nesterov, Y. and Polyak, B.~T.
\newblock Cubic regularization of newton method and its global performance.
\newblock \emph{Mathematical Programming}, 108\penalty0 (1):\penalty0 177--205,
  2006.

\bibitem[Pearlmutter(1994)]{pearlmutter1994fast}
Pearlmutter, B.~A.
\newblock Fast exact multiplication by the hessian.
\newblock \emph{Neural computation}, 6\penalty0 (1):\penalty0 147--160, 1994.

\bibitem[Reddi et~al.(2017)Reddi, Zaheer, Sra, Poczos, Bach, Salakhutdinov, and
  Smola]{reddi2017generic}
Reddi, S.~J., Zaheer, M., Sra, S., Poczos, B., Bach, F., Salakhutdinov, R., and
  Smola, A.~J.
\newblock A generic approach for escaping saddle points.
\newblock \emph{arXiv preprint arXiv:1709.01434}, 2017.

\bibitem[Simchowitz et~al.(2017)Simchowitz, Alaoui, and
  Recht]{simchowitz2017gap}
Simchowitz, M., Alaoui, A.~E., and Recht, B.
\newblock On the gap between strict-saddles and true convexity: An omega (log
  d) lower bound for eigenvector approximation.
\newblock \emph{arXiv preprint arXiv:1704.04548}, 2017.

\bibitem[Xu et~al.(2017)Xu, Roosta-Khorasani, and Mahoney]{xu2017newton}
Xu, P., Roosta-Khorasani, F., and Mahoney, M.~W.
\newblock Newton-type methods for non-convex optimization under inexact hessian
  information.
\newblock \emph{arXiv preprint arXiv:1708.07164}, 2017.

\bibitem[Xu \& Yang(2017)Xu and Yang]{xu2017first}
Xu, Y. and Yang, T.
\newblock First-order stochastic algorithms for escaping from saddle points in
  almost linear time.
\newblock \emph{arXiv preprint arXiv:1711.01944}, 2017.

\bibitem[Zhang et~al.(2015)Zhang, Lee, Wainwright, and
  Jordan]{zhang2015learninghs}
Zhang, Y., Lee, J.~D., Wainwright, M.~J., and Jordan, M.~I.
\newblock Learning halfspaces and neural networks with random initialization.
\newblock \emph{arXiv preprint arXiv:1511.07948}, 2015.

\bibitem[Zhang et~al.(2017)Zhang, Liang, and Charikar]{Zhang2017AHT}
Zhang, Y., Liang, P., and Charikar, M.
\newblock A hitting time analysis of stochastic gradient langevin dynamics.
\newblock In \emph{COLT}, 2017.

\end{thebibliography}
